\documentclass{article}

\usepackage{microtype}
\usepackage{graphicx}
\usepackage{subcaption}
\usepackage{booktabs} 
\usepackage{placeins}

\usepackage{hyperref}

\usepackage[accepted]{icml2026}

\usepackage{amsmath,amssymb,amsthm,mathtools,bm,mathrsfs}
\usepackage{microtype}
\usepackage{enumitem}

\usepackage[capitalize,noabbrev,nameinlink]{cleveref}
\crefformat{equation}{(#2#1#3)}
\crefrangeformat{equation}{(#3#1#4) to~(#5#2#6)}
\crefmultiformat{equation}{(#2#1#3)}{ and~(#2#1#3)}{, (#2#1#3)}{ and~(#2#1#3)}
\Crefformat{equation}{Equation~(#2#1#3)}
\Crefrangeformat{equation}{Equations~(#3#1#4) to~(#5#2#6)}
\Crefmultiformat{equation}{Equations~(#2#1#3)}{ and~(#2#1#3)}{, (#2#1#3)}{ and~(#2#1#3)}

\newtheorem{theorem}{Theorem}[section]
\newtheorem{proposition}[theorem]{Proposition}
\newtheorem{lemma}[theorem]{Lemma}
\newtheorem{corollary}[theorem]{Corollary}
\theoremstyle{definition}
\newtheorem{definition}[theorem]{Definition}
\newtheorem{remark}[theorem]{Remark}

\newcommand{\R}{\mathbb{R}}
\newcommand{\E}{\mathbb{E}}
\newcommand{\Var}{\operatorname{Var}}

\newcommand{\Exp}{\operatorname{Exp}}

\newcommand{\cO}{O}
\newcommand{\so}{\mathrm{o}}
\newcommand{\norm}[1]{\left\lVert #1 \right\rVert}
\newcommand{\ip}[2]{\left\langle #1, #2 \right\rangle}
\newcommand{\Tub}{\operatorname{Tub}}
\newcommand{\vol}{\operatorname{Vol}}
\newcommand{\reach}{\operatorname{reach}}

\newcommand{\Ric}{\operatorname{Ric}}
\newcommand{\tr}{\operatorname{tr}}

\newcommand{\Id}{\operatorname{Id}}

\usepackage[textsize=tiny]{todonotes}

\icmltitlerunning{Rao-Blackwellized Score Matching on Manifolds}

\begin{document}

\twocolumn[
 \icmltitle{Rao-Blackwellized Score Matching on Manifolds}



\icmlsetsymbol{equal}{*}

\begin{icmlauthorlist}
 \icmlauthor{Divit Rawal}{berk}
\end{icmlauthorlist}

 \icmlaffiliation{berk}{Department of Statistics, University of California, Berkeley, Berkeley CA, U.S.A}

 \icmlcorrespondingauthor{Divit Rawal}{divit.rawal@berkeley.edu}

 \icmlkeywords{Machine Learning, ICML, denoising score matching, Rao-Blackwellization, Riemannian score, geometric statistics, manifold hypothesis}

\vskip 0.3in ]



\printAffiliationsAndNotice{} 

\begin{abstract}
We study denoising score matching (DSM) when data are drawn from an embedded manifold $M \subset \mathbb{R}^D$. We show that under ambient Gaussian corruption, the target has variance that diverges as the noise scale decreases and correct for it by regressing against the conditional expectation given the nearest point projection on the manifold: the $L^2$-optimal Rao-Blackwellized target. We then compute the small-noise expansion of this target and show that it recovers the true intrinsic Riemannian score to first order, with a second-order bias from a Tweedie term and two geometric terms dependent on how the manifold is embedded in ambient space: a curvature operator acting on the intrinsic score, and an additive drift generated by the spatial variation of the embedding's second fundamental form. On hyperspheres, we derive a simplified formula and show that both geometric terms vanish exactly on $S^2$, offering a theoretical explanation for why ambient DSM performs comparably to intrinsic methods on real Earth science spherical data in prior work.
\end{abstract}

\section{Introduction}

The manifold hypothesis --- that high-dimensional data concentrates on or near a lower-dimensional submanifold $M \subset \R^D$ --- underpins much of modern generative modeling. For score-based generative models in particular, it creates a tension: the latent law $Z\sim q\ d\vol_M$ defined to be positive only on the manifold is singular with respect to the ambient Lebesgue measure, so the ambient score $\nabla\log q$ is not defined, and score matching is not well-defined. A common alternative is denoising score matching (DSM), in which isotropic Gaussian noise (with variance $\sigma^2$) is used to corrupt points on the manifold so that the score in ambient space is defined. However, as a result, the DSM regression learns only a $\sigma$-dependent surrogate of the intrinsic Riemannian score instead of the true score.

A central question therefore remains unresolved: \emph{What object does ambient DSM actually learn on manifold-supported data?}

Existing approaches address this issue in two different ways. Intrinsic methods \citep{debortoli2022riemannianscorebasedgenerativemodelling,huang2022riemanniandiffusionmodels} replace ambient Gaussian corruption by intrinsic manifold diffusions, such as Riemannian Brownian motion, and directly estimate tangent vector fields on $M$. These approaches avoid the ambient singularity but require manifold-specific geometric infrastructure, including exponential maps and heat-kernel simulation, which are often computationally expensive. Ambient approaches \citep{levyjurgenson2026manifoldawaredenoisingscore} instead retain the standard Euclidean DSM methodology and rely on the heuristic that, after projection onto the tangent bundle, the learned field converges to the intrinsic Riemannian score as $\sigma \to 0^+$. Both lines have yielded impressive empirical results, but both also recognize that the manifold hypothesis causes ambient DSM targets to carry some additional noise and that generalization bounds degrade as $\sigma \to 0^+$ \citep{yakovlev2025generalizationerrorbounddenoising}.

In this work, we work under the ambient Gaussian corruption model:
\[
X = Z + \sigma \xi,
\qquad
Z \sim q \, d\mathrm{Vol}_M,
\qquad
\xi \sim \mathcal{N}(0,I_D),
\]
where $Z$ is sampled from the probability law defined on the manifold $M$ and $X$ is the corresponding noisy point. Our results make heavy use of the projection of the usual DSM target $(Z-X)/\sigma^2$ onto the nearest point on the manifold, which we shall denote as $\pi(X)$. After this projection, the target still carries noise from the directions orthogonal to the manifold (normal-fiber noise) which diverges at rate $d/\sigma^2$ as $\sigma \to 0^+$. We show that by taking the expectation of the projection conditioned on the event $\pi(X)=z$, one obtains the $L^2$-optimal predictor of the raw tangent target among all estimators that depend on $X$ only through $\pi(X)$; a Rao-Blackwellization of the DSM regression problem against the projection. We then compute its expansion and derive its exact bias.

Concretely, our contributions are the following:
\begin{enumerate}
    \item We identify the Rao-Blackwellized tangent target $r_\sigma$ and prove that it is the unique $L^2$-optimal predictor of the raw DSM target among estimators that see the observation $X$ only through the projection $\pi(X)$. We then show that the raw target's variance diverges at rate $d/\sigma^2$ while that of $r_\sigma$ stays bounded, and that $d/\sigma^2$ is an irreducible Bayes-risk floor within this class of predictors.

    \item We compute the order-$\sigma^2$ expansion of $r_\sigma$ on arbitrary embedded submanifolds, and quantify the bias from an intrinsic Tweedie term and two extrinsic terms: a curvature operator built from the Weingarten and Ricci operators of the embedding, acting on the intrinsic score, and an additive score-independent drift generated by the spatial variation of the second fundamental form. The latter vanishes exactly on embeddings with parallel second fundamental form, including affine subspaces and round spheres, but is nonzero in general --- so ambient corruption can produce a tangent drift even where the intrinsic score is zero.

    \item On flat supports the construction reduces exactly to ordinary lower-dimensional Gaussian DSM; on spheres we derive closed-form curvature coefficients and find an exact cancellation on $S^2$, justifying the empirical success of ambient DSM on Earth science data.
\end{enumerate}

\section{Relation to Prior Work}
\label{sec:related}

\paragraph{Score matching and denoising score matching.}
Score matching was introduced by \citet{hyvarinen05a} as a way of fitting non-normalized statistical models without evaluating normalizing constants. \citet{10.1162/NECO_a_00142} established the equivalence between denoising score matching (DSM) at a fixed noise level $\sigma$ and regression of the denoising residual $(Z-X)/\sigma^2$, via Tweedie's formula \citep{Efron2011TweediesFA,robbins1956empirical}. Both derivations assume that the latent law $q$ admits a density with respect to ambient Lebesgue measure. When $q$ is supported on a lower-dimensional submanifold, the ambient score $\nabla \log q$ is not defined and DSM is only meaningful at positive $\sigma$. Our results clarify what DSM is actually estimating in this regime: the raw tangent denoising target $T_\sigma$ contains a nuisance normal-fiber component whose conditional variance diverges at rate $d/\sigma^2$, and its projection onto $\pi(X)$ isolates the signal-bearing part.

\paragraph{Score-based generative models and the manifold hypothesis.}
The diffusion-model literature \citep{song2020generativemodelingestimatinggradients,ho2020denoisingdiffusionprobabilisticmodels,song2021scorebasedgenerativemodelingstochastic} formulates generation as reversing a Gaussian noising process, which implicitly regularizes any singular latent law by Gaussian convolution. Theoretical work on this regime has repeatedly observed that the ambient score blows up near the support of a low-dimensional latent. \citet{pidstrigach2022scorebasedgenerativemodelsdetect} showed that score-based generative models trained on data concentrated near a manifold recover drift fields that align with the normal direction, effectively detecting the manifold. \citet{debortoli2023convergencedenoisingdiffusionmodels} gave quantitative convergence guarantees for denoising diffusion models under a manifold hypothesis, and \citet{debortoli2022riemannianscorebasedgenerativemodelling} introduced Riemannian score-based generative modelling (RSGM), which replaces the ambient Gaussian noising process by an intrinsic heat-kernel diffusion on a known manifold. RSGM estimates an intrinsic Riemannian score directly, but requires the ability to simulate the heat kernel on $M$. Our work is complementary: we ask what the ambient DSM target actually identifies when the latent law is singular with respect to Lebesgue measure, and show that a single Rao-Blackwellization step against the nearest-point projection $\pi(X)$ converts the ambient target into an $\cO(\sigma^2)$-accurate estimator of the intrinsic Riemannian score, without simulating any heat kernel.

\paragraph{Sample complexity and rates.}
\citet{chen2023samplingeasylearningscore} and \citet{oko2023diffusionmodelsminimaxoptimal} derive sample-complexity and minimax-rate guarantees for diffusion models, working in the ambient formulation. Our results are population-level identification theorems rather than rate theorems, but the variance collapse in \cref{thm:variance-collapse} and the finite-sample bound in \cref{prop:finite-sample} quantify the price of ignoring the manifold structure. They suggest that Rao-Blackwellization can yield an asymptotically stronger signal-to-noise improvement as $\sigma\to0^+$, rather than merely a constant-factor gain.

\paragraph{Manifolds and Geometry.}
The tubular-neighborhood and positive-reach calculus we use was developed by \citet{federer1959curvature}; we use it in the form developed by \citet{10.5555/3116660.3117013} for manifold learning. The local-averaging regression proposition in \cref{app:finite-sample} is based on the classical rate theory of \citet{Fan1992DesignadaptiveNR}.

Existing DSM theory assumes an absolutely continuous latent law \citep{hyvarinen05a,10.1162/NECO_a_00142}; the manifold-hypothesis literature has observed blowup of the ambient score \citep{pidstrigach2022scorebasedgenerativemodelsdetect,debortoli2023convergencedenoisingdiffusionmodels}; and the intrinsic heat-kernel route \citep{debortoli2022riemannianscorebasedgenerativemodelling} bypasses the ambient score entirely. In this work, we provide three pieces absent from prior analyses: (i) a canonicality statement identifying $\pi(X)$ as the finest fiber-collapsing summary; (ii) the exact constant $d/\sigma^2$ in the raw-target variance, together with a matching Bayes-risk floor; and (iii) an exact equality reduction of ambient DSM to lower-dimensional DSM in the flat case, which pins down the baseline against which curvature effects must be measured.

\section{Setup}
\label{sec:setup}

Before proceeding, we define some notation. We also direct the reader to \cref{app:notation} for a brief refresher on the Riemannian geometry used to derive our results; a more thorough treatment may be found in \citet{docarmo1992riemannian}.

Let $M \subset \R^D$ be a compact embedded $C^6$ submanifold of dimension $d$ and positive reach. Let $q \in C^6(M)$ be strictly positive with respect to the Riemannian volume measure $d\vol_M$.\footnote{It is worth noting that $C^5$ suffices for every statement below except the $\cO(\sigma^4)$ remainders of \cref{sec:extrinsic}, in which we need six derivatives only because the induced metric in graph coordinates is built from $\partial h$ and so loses one derivative relative to the embedding, see \cref{app:uniformity} for technical details.} Unless otherwise specified, all norms are taken to be the standard $L^2$ norm. Following standard notation, we use a.s.\ and a.e.\ to mean almost surely and almost everywhere respectively.

We work under the Gaussian corruption model
\begin{equation}
 Z \sim q\, d\vol_M,
 \quad
 X = Z + \sigma \xi,
 \quad
 \xi \sim \mathcal{N}(0,I_D).
 \label{eq:model}
\end{equation}

Let $r_0<\reach(M)$ and write $\Tub_{r_0}(M)$ for the corresponding tubular neighborhood. For sufficiently small $\sigma$, the event \[ \mathcal{E}_\sigma \doteq \{X\in \Tub_{r_0}(M)\} \] has probability at least $1-\cO(\exp(-c/\sigma^2))$ for some $c>0$. On $\mathcal{E}_\sigma$, the nearest-point projection \[ \pi:\Tub_{r_0}(M)\to M \] is well-defined and smooth. Throughout we work on $\mathcal{E}_\sigma$; all omitted tails are exponentially small in $\sigma^{-2}$ and do not affect any polynomial-order expansions we give.

For $z\in M$, let $P_T(z)$ and $P_N(z)$ denote the orthogonal projections onto $T_zM$ and $N_zM$ (the tangent and normal spaces, respectively).

The standard denoising target is
\begin{equation}
 Y_\sigma \doteq \frac{Z-X}{\sigma^2}.
 \label{eq:denoising-target}
\end{equation}
We define the raw tangent denoising target
\begin{equation}
 T_\sigma \doteq P_T(\pi(X))Y_\sigma.
 \label{eq:raw-tangent-target}
\end{equation}
Since $X-\pi(X)\in N_{\pi(X)}M$, we may write
\begin{equation}
 T_\sigma
 =
 \frac{1}{\sigma^2}P_T(\pi(X))(Z-\pi(X)).
 \label{eq:tangent-target-simplified}
\end{equation}

\begin{definition}[Rao-Blackwellized tangent target]
For $z\in M$, define
\begin{equation}
 r_\sigma(z)
 \doteq
 \E[T_\sigma\mid \pi(X)=z]
 \in T_zM,
 \label{eq:rb-target}
\end{equation}
yielding a measurable tangent field $r_\sigma: M \to TM$ with $r_\sigma(z)\in T_zM$ by construction. For any measurable tangent field $h: M\to TM$, define the projected denoising risk
\begin{equation}
 \mathcal{R}_\sigma(h)
 \doteq
 \E \norm{T_\sigma-h(\pi(X))}^2.
 \label{eq:projected-risk}
\end{equation}
\end{definition}

\paragraph{Conditioning on $\pi(X)=z$.}
The set $\{\pi(X)=z\}$ has probability zero for each individual $z\in M$, so the conditional expectation in \cref{eq:rb-target} is interpreted as a regular conditional expectation. The Federer-Gray tube formula together with the Gaussian density of $X$ shows that $\pi(X)$ admits a smooth strictly positive density on $M$ with respect to $d\vol_M$ on the event $\mathcal{E}_\sigma$ (the explicit form is computed in \cref{app:fiber-posterior}). Throughout, equalities of the form $r_\sigma(\pi(X))=\cdots$ and $\E[\,\cdot\mid \pi(X)=z]=\cdots$ are understood to hold $\pi(X)$-a.s., equivalently $d\vol_M$-a.e.\ on $M$.

\section{Canonicality of $r_\sigma$}
\label{sec:canonicality}

Here, we show that within the family of fiber-measurable tangent fields, the Rao-Blackwellized target $r_\sigma$ is the unique $L^2$-risk minimizer; a straightforward Pythagorean decomposition that also gives canonicality of $\pi(X)$ among fiber-collapsing summaries.\footnote{A fiber-collapsing summary is one which maps the set of all points sharing the same unique closest point on a submanifold to that closest point.} We also show that $r_\sigma$ agrees with the intrinsic Riemannian score $\nabla_M\log q$ to leading order in $\sigma$ and that the residual risk of the raw target $T_\sigma$ against any fiber-collapsing predictor diverges at the exact rate $d/\sigma^2$, an irreducible Bayes-risk lower bound. Full proofs of the claims are in \Cref{app:leading-order-proof}.

\subsection{$L^2$ Optimality}
\label{sec:rb-identities}

The risk $\mathcal{R}_\sigma(h)$ in \cref{eq:projected-risk} penalizes a tangent field $h$ for its $L^2$ distance to $T_\sigma$. Because $r_\sigma(\pi(X))$ is a conditional expectation, an $L^2$-projection decomposition applies.

\begin{theorem}[Projected denoising risk]
\label{thm:rb-risk}
For every measurable tangent field $h:M\to TM$,
\begin{equation}
 \mathcal{R}_\sigma(h)
 =
 \mathcal{R}_\sigma(r_\sigma)
 +
 \E \norm{r_\sigma(\pi(X))-h(\pi(X))}^2.
 \label{eq:pythagorean}
\end{equation}
In particular, $r_\sigma$ is the unique (up to null sets) minimizer of $\mathcal{R}_\sigma$.
\end{theorem}

\begin{proof}
Since $r_\sigma(\pi(X))=\E[T_\sigma\mid \pi(X)]$, we may write
\begin{align*}
T_\sigma-h(\pi(X))
&=
\bigl(T_\sigma-r_\sigma(\pi(X))\bigr)\\
&\quad+
\bigl(r_\sigma(\pi(X))-h(\pi(X))\bigr).
\end{align*}
The second term is $\sigma(\pi(X))$-measurable, whereas the first is
orthogonal in $L^2$ to every $\sigma(\pi(X))$-measurable square-integrable
random variable. Hence the cross term in
$\E\norm{T_\sigma-h(\pi(X))}^2$ vanishes, which is \cref{eq:pythagorean}.
Uniqueness is immediate: the second summand is nonnegative and vanishes
iff $h(\pi(X))=r_\sigma(\pi(X))$ a.s.
\end{proof}

The same identity, applied with $h$ replaced by an estimator measurable with respect to a coarser statistic, extends the Pythagorean decomposition to the full family of fiber-collapsing summaries. First note that there exist statistics $S = S(X)$ satisfying $\sigma(S) \subseteq \sigma(\pi(X))$; equivalently, there exists a measurable function $\bar{S}$ such that $S = \bar{S}(\pi(X))$ a.s.
For any such $S$, define $\eta_S \doteq \mathbb{E}[T_\sigma \mid S]$. The tower property of expectations gives
\begin{equation}
\eta_S=\E[T_\sigma\mid S]=\E\!\bigl[r_\sigma(\pi(X))\mid S\bigr],
\label{eq:canonicality}
\end{equation}
so the optimal $S$-measurable predictor of $T_\sigma$ equals the optimal $S$-measurable predictor of $r_\sigma(\pi(X))$, and the excess risk of any $S$-measurable estimator $\eta$ over $r_\sigma(\pi(X))$ decomposes exactly as
\begin{equation}
\begin{split}
\E\norm{T_\sigma-\eta}^2
&=\E\norm{T_\sigma-r_\sigma(\pi(X))}^2\\
&\quad+\E\norm{r_\sigma(\pi(X))-\eta_S}^2\\
&\quad+\E\norm{\eta_S-\eta}^2.
\end{split}
\label{eq:three-term}
\end{equation}
Among estimators measurable with respect to fiber-collapsing summaries, $r_\sigma(\pi(X))$ is therefore the unique minimum-risk choice, and \cref{eq:three-term} identifies the exact cost $\E\norm{r_\sigma(\pi(X))-\eta_S}^2$ of coarsening $\pi(X)$ to $S$. This is the Rao-Blackwell statement for our setup: among fiber-collapsing statistics, $\pi(X)$ is the canonical finest summary and $r_\sigma(\pi(X))$ is the corresponding $L^2$-optimal predictor of $T_\sigma$.

\subsection{Leading-Order Behavior and Variance}
\label{sec:leading-order}

Having identified the correct object to target, we now ask how close $r_\sigma$ is to the intrinsic Riemannian score. A tubular-coordinate Bayes calculation combined with a manifold extension of Stein's identity (\Cref{app:leading-order-proof}) yields, uniformly in $z\in M$,
\begin{equation}
r_\sigma(z)=\nabla_M \log q(z)+\cO(\sigma^2),
\qquad
\sigma\to 0^+.
\label{eq:main-intrinsic}
\end{equation}
This is a small-noise \emph{population} statement; a finite-sample consistency rate in the number of observations used to form the conditional expectation is given in \cref{app:finite-sample}. The $\sigma^2$ residual is nonzero on a curved support and is precisely the $\sigma^2$ coefficient computed in \cref{sec:extrinsic}.

We now quantify how much of the raw target's risk is supervision noise rather than signal. Throughout, for a vector-valued random variable $V$, the symbol $\Var(V)$ denotes the trace covariance, equivalently $\E\norm{V-\E V}^2$.

\begin{theorem}[Variance collapse under Rao-Blackwellization]
\label{thm:variance-collapse}
Uniformly in $z\in M$,
\begin{equation}
\Var(T_\sigma\mid \pi(X)=z)=\frac{d}{\sigma^2}+\cO(1).
\label{eq:cond-var-blowup}
\end{equation}
Consequently,
\begin{equation}
\Var(T_\sigma)
=
\Var(r_\sigma(\pi(X)))
+
\frac{d}{\sigma^2}
+
\cO(1),
\label{eq:uncond-var-blowup}
\end{equation}
so the raw tangent denoising target has variance diverging like
$d/\sigma^2$ while the Rao-Blackwellized target has $\cO(1)$ variance.
\end{theorem}

The constant $d/\sigma^2$ is the Bayes-risk floor for regression against $T_\sigma$ \emph{within the restricted class of fiber-collapsing summaries}: estimators of $T_\sigma$ that depend on the observation $X$ only through some $S(X)$ with $\sigma(S)\subseteq\sigma(\pi(X))$.\footnote{It is easy to see that unrestricted estimators of $T_\sigma$ can trivially attain zero risk by taking $\eta = T_\sigma$ itself; the floor reflects the cost of predicting $T_\sigma$ from any summary that does not see the full normal-fiber noise.} Indeed, for any such $S$, \cref{eq:three-term} bounds
\begin{align*}
\inf_{\eta\in L^2(\sigma(S))}\E\norm{T_\sigma-\eta}^2
&\ge \E\norm{T_\sigma-r_\sigma(\pi(X))}^2\\
&=\E\!\left[\Var(T_\sigma\mid \pi(X))\right]\\
&=\frac{d}{\sigma^2}+\cO(1),
\end{align*}
where the last equality is the tower-rule average of \cref{eq:cond-var-blowup}. Equality at leading order holds only when $\sigma(S)=\sigma(\pi(X))$ (modulo null sets). The roles of $\pi$ and $r_\sigma$ are distinct: $\pi(X)$ is the canonical finest fiber-collapsing summary, and $r_\sigma$ is the $L^2$-optimal predictor of $T_\sigma$ given that summary. Every fiber-collapsing $S$ inherits the $d/\sigma^2$ floor. \Cref{fig:variance-collapse} demonstrates the variance-collapse rate of \Cref{thm:variance-collapse} on $S^2$ with the von Mises-Fisher density.

\begin{figure}[t]
\centering
\includegraphics[width=\linewidth]{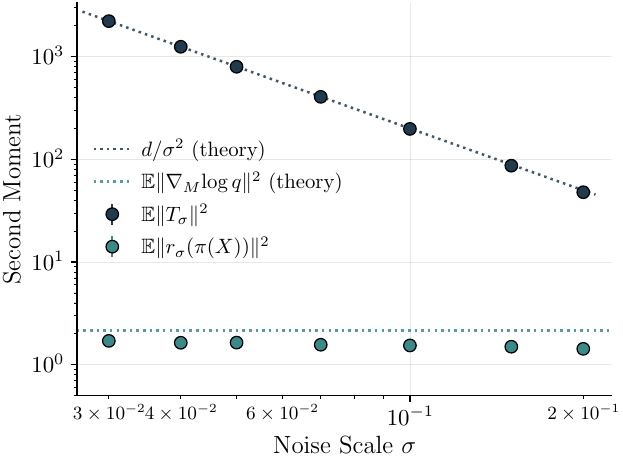}
\caption{Variance collapse on $S^2$ under $\mathrm{vMF}(\mu,\kappa{=}2)$.
Second moment of the raw target $T_\sigma$ (black, slope $-2$ in
$\log\sigma$, matching $d/\sigma^2$ with $d=2$) versus the
Rao-Blackwellized target $r_\sigma(\pi(X))$ (blue, flat at the
theoretical $\E\norm{\nabla_M\log q}^2$). The gap is the claimed $d/\sigma^2$ Bayes-risk floor.}
\label{fig:variance-collapse}
\end{figure}

\section{Extrinsic $\sigma^2$ Correction}
\label{sec:extrinsic}

We now provide a closed-form expansion of the canonical target $r_\sigma$ in $\sigma$, with two extrinsic curvature corrections that are absent from intrinsic-noising analyses.

\subsection{Affine Case}
\label{sec:flat}

Before proceeding to general manifolds, let us gain some intuition by treating the affine case where all identities are exact.

Let $V\subset\R^D$ be an affine $d$-plane, $P$ the orthogonal projection onto $V$, and $Q\doteq I-P$. We still work under our ambient Gaussian corruption setup, and write $T\doteq PX$, $N\doteq QX$ so that $T=Z+\sigma P\xi$ and $N=\sigma Q\xi$ are independent. For $\phi_\sigma^{(d)}$ the probability density function of a $d$-dimensional isotropic Gaussian with variance $\sigma^2$, let \[ p_T\doteq q*\phi_\sigma^{(d)} \] be the $d$-dimensional Gaussian convolution of $q$ on $V$.

\begin{proposition}
\label{prop:flat}
In the setting above, the ambient score factorizes as
\begin{equation}
\nabla\log p_\sigma(x)
=
P\nabla_V\log p_T(Px)-\frac{Qx}{\sigma^2},
\label{eq:flat-score}
\end{equation}
the Rao-Blackwellized target coincides with the flat Tweedie score along $V$,
\begin{equation}
r_\sigma(t)=\nabla_V\log p_T(t),
\qquad
p_T=q*\phi_\sigma^{(d)},
\label{eq:flat-rb}
\end{equation}
and ambient DSM restricted to the tangent channel reduces exactly to $d$-dimensional Gaussian DSM on $V$, that is, for every measurable $h:V\to V$, with $g_h(x)\doteq Ph(Px)-Qx/\sigma^2$,
\begin{equation}
\E\norm{Y_\sigma-g_h(X)}^2
=
\E\left\|\tfrac{Z-T}{\sigma^2}-h(T)\right\|^2.
\label{eq:flat-dsm-reduction}
\end{equation}
\end{proposition}

\begin{proof}
Write
\[
Y_\sigma
=
\frac{Z-X}{\sigma^2}
=
\frac{Z-T}{\sigma^2}
-
\frac{N}{\sigma^2},
\qquad
g_h(X)=h(T)-\frac{N}{\sigma^2}.
\]
Therefore
\[
Y_\sigma-g_h(X)
=
\frac{Z-T}{\sigma^2}-h(T),
\]
which lies in $V$. Taking squared norms and expectations gives
\cref{eq:flat-dsm-reduction}.
\end{proof}

Thus, on a flat support, the intrinsic score, the $V$-component of the ambient score, and the Rao-Blackwellized tangent target all coincide; no $\sigma^2$ correction is needed.

\subsection{General Case}
\label{sec:curved-expansion}

We now state the main expansion of $r_\sigma$ on a curved support. The coefficient decomposes into an intrinsic flat-Tweedie term, an extrinsic curvature term acting on the intrinsic score, and an additive extrinsic drift generated by the spatial variation of the embedding.

\begin{theorem}[Extrinsic $\sigma^2$ correction]
\label{thm:extrinsic}
Under the assumptions of \cref{sec:setup}, uniformly in $z\in M$,
\begin{equation}
\begin{split}
 r_\sigma(z)
 &=
 \nabla_M\log q(z)\\
 &\quad+
 \sigma^2\bigl[b_q(z)+g_M^{\mathrm{ext}}(z)+g_M^{\mathrm{inh}}(z)\bigr]
 +
 \cO(\sigma^4),
\end{split}
 \label{eq:extrinsic-main}
\end{equation}
as $\sigma\to 0^+$,
where the intrinsic flat-Tweedie term is
\begin{equation}
 b_q(z)
 =
 \tfrac{1}{2}\nabla_M\!\left[\Delta_M \log q + \norm{\nabla_M \log q}^2\right](z),
 \label{eq:bq-main}
\end{equation}
the extrinsic curvature term is
\begin{equation}
 g_M^{\mathrm{ext}}(z)
 =
 \Bigl(\tfrac{1}{2}W_{H(z)}-\Ric^\sharp_z\Bigr)\!\bigl(\nabla_M \log q(z)\bigr),
 \label{eq:gM-ext-main}
\end{equation}
and the curvature-inhomogeneity term is
\begin{equation}
 g_M^{\mathrm{inh}}(z)
 =
 \tfrac{1}{2}\mathcal C_M(z).
 \label{eq:gM-inh-main}
\end{equation}
Here $W_u:T_zM\to T_zM$ is the Weingarten operator in normal direction $u$,
$H(z)=\sum_{i=1}^d \mathrm{II}_z(e_i,e_i)\in N_zM$ is the mean curvature vector
of $M\hookrightarrow\R^D$ at $z$, $\Ric^\sharp_z$ is the Ricci endomorphism
of $(M,g_M)$ at $z$, and $\mathcal C_M(z)\in T_zM$ is the curvature-inhomogeneity
vector \cref{eq:CM-def}, a contraction of $\mathrm{II}_z$ against
$\nabla^\perp\mathrm{II}_z$. The two extrinsic terms are structurally different:
$g_M^{\mathrm{ext}}$ is a curvature operator applied to the intrinsic score and
vanishes wherever that score vanishes, whereas $g_M^{\mathrm{inh}}$ is
score-independent and vanishes exactly when the embedding is parallel at $z$
($\nabla^\perp\mathrm{II}_z=0$). In the flat case $M=V$ one has
$\mathrm{II}\equiv 0$, so $W_H$, $\Ric^\sharp$ and $\mathcal C_M$ all vanish,
$g_M^{\mathrm{ext}}\equiv g_M^{\mathrm{inh}}\equiv 0$, and
\cref{eq:extrinsic-main} reduces to the flat Tweedie identity
\cref{eq:flat-rb}.
\end{theorem}

Examining \cref{thm:extrinsic}, it is clear that ambient Gaussian corruption does not merely smooth the intrinsic density on $M$; it distorts the target score through a curvature-dependent embedding term, and it additionally injects a drift that does not depend on the score at all. Thus, even after projecting onto the tangent bundle and Rao-Blackwellizing away the singular fiber noise, ambient DSM is biased relative to intrinsic noising by the explicit $\sigma^2(g^{\mathrm{ext}}_M+g^{\mathrm{inh}}_M)$ term. The proof is a graph-coordinate Bayes calculation with two geometric factors --- the induced volume-form correction and the tube-Jacobian mean-curvature correction. Their quadratic jets combine without cancellation into the operator $\tfrac12 W_{H(z)}-\Ric^\sharp_z$ acting on the score, and their cubic jets, which are linear in $\nabla^\perp\mathrm{II}$, combine into the score-independent vector $\tfrac12\mathcal C_M(z)$. We defer the full proof to \cref{app:second-order}.

\begin{remark}[Comments on \cref{eq:gM-ext-main,eq:gM-inh-main}]\label{rem:comments} We note that (i) the extrinsic term is invisible to any analysis that corrupts $Z$ by intrinsic manifold noise (e.g. Brownian motion on $M$ or geodesic noising), where only $b_q(z)$ appears. The term $g_M^{\mathrm{ext}}(z)$ is a result of ambient Gaussian corruption and vanishes identically in intrinsic-noising analyses. (ii) Using the Gauss equation $\Ric^\sharp_z=W_{H(z)}-\mathcal S_z$ with $\mathcal S_z \doteq \sum_\alpha W_{n_\alpha}^2$ for any orthonormal normal frame, one can equivalently write $g_M^{\mathrm{ext}}(z)=(\mathcal S_z-\tfrac12 W_{H(z)})\nabla_M\log q(z)$. (iii) The same holds for $g_M^{\mathrm{inh}}$, which is likewise a pure artifact of ambient corruption; unlike $g_M^{\mathrm{ext}}$ it is not a function of the intrinsic score, and unlike $\Ric^\sharp$ it is not an intrinsic invariant of $(M,g_M)$. The cylinder over a plane curve of curvature $\kappa$ is intrinsically flat, so $b_q$ and $\Ric^\sharp$ see nothing, yet $g_M^{\mathrm{inh}}=\kappa\kappa'e_1\ne 0$ wherever $\kappa'\ne0$.
\end{remark}

\begin{figure}[t]
\centering
\includegraphics[width=\linewidth]{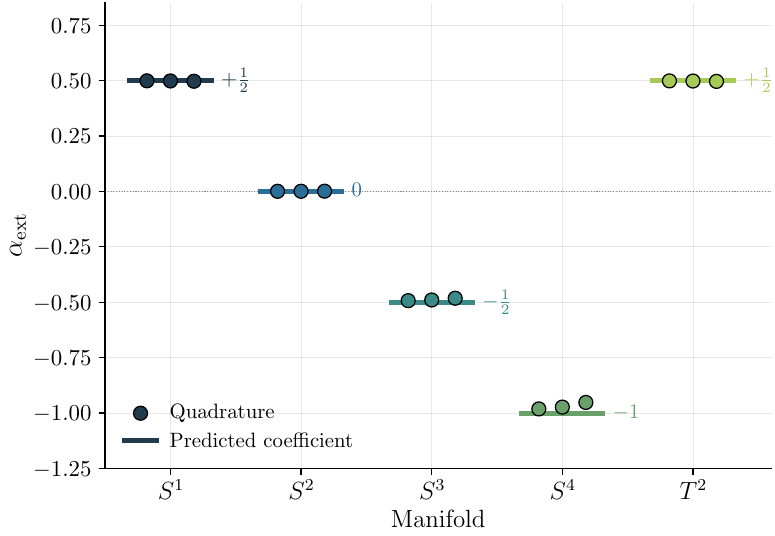}
\caption{Extrinsic coefficient $\alpha_{\mathrm{ext}}$ across manifolds.
Quadrature estimates (one dot per $\sigma\in\{0.05,0.06,0.08\}$) are
computed by Gauss-Hermite quadrature of $r_\sigma(z)$ and compared to
the predicted coefficients $\alpha_d=1-d/2$ on $S^d$ and
$\alpha=+\tfrac12$ on $T^2$. Latent density $q$ is von Mises-Fisher
with $\kappa=2$ on each $S^d$ and a wrapped Gaussian with $\kappa=1.5$
on $T^2$. Each manifold shown has parallel second fundamental form, so
$g^{\mathrm{inh}}_M\equiv 0$ (\cref{eq:gM-inh-main}) and the measurement isolates
$g^{\mathrm{ext}}_M$.}
\label{fig:extrinsic-bias}
\end{figure}

\Cref{fig:extrinsic-bias} computes $r_\sigma(z)$ numerically, subtracts the intrinsic score and the Tweedie term, and divides by $\sigma^2\nabla_M\log q(z)$ to isolate the dimensionless extrinsic coefficient $\alpha_{\mathrm{ext}}$ we predict. The manifolds shown are parallel ($\nabla^\perp\mathrm{II}\equiv0$), so $g^{\mathrm{inh}}_M$ vanishes there.

\subsection{Hyperspheres}
\label{sec:sphere}

We record the specialization of \cref{eq:extrinsic-main} to the unit $d$-sphere $S^d\subset\R^{d+1}$, a common application of ambient DSM.

\begin{corollary}[Extrinsic coefficient on $S^d$]
\label{cor:sphere-coefficient}
For $M=S^d\subset\R^{d+1}$ with the round metric and outward normal
$\nu=z$, $W_\nu=-\Id_{T_zM}$, $H(z)=-dz$,
$W_{H(z)}=d\Id_{T_zM}$, and $\Ric^\sharp_z=(d-1)\Id_{T_zM}$, so
\begin{equation}
 g_{S^d}^{\mathrm{ext}}(z)
 =
 \bigl(1-\tfrac{d}{2}\bigr)\nabla_M\log q(z).
 \label{eq:sphere-coeff}
\end{equation}
In particular, the scalar multiplier $\alpha_d \doteq 1-d/2$ attached to
$\nabla_M\log q(z)$ in the $\sigma^2$ correction is identically 0 when $d=2$.
Moreover the round embedding is parallel, $\nabla^\perp\mathrm{II}\equiv 0$, so
$\mathcal C_{S^d}\equiv 0$ and $g^{\mathrm{inh}}_{S^d}\equiv 0$: the complete
$\sigma^2$ coefficient on $S^d$ is $b_q+(1-\tfrac d2)\nabla_M\log q$.
\end{corollary}

We note that the $S^d$ specialization is unique because $S^d$ is an Einstein manifold: the Ricci curvature tensor is proportional to the metric tensor. In particular, on $S^2$, $\tfrac12 W_H=\Ric^\sharp=\Id$ forces the two curvature contributions to cancel; since the round embedding is also parallel, the entire extrinsic $\sigma^2$ bias vanishes on $S^2$. Proof is a direct substitution; see \cref{app:sphere-d}.

\begin{figure*}[t]
\centering
\includegraphics[width=\linewidth]{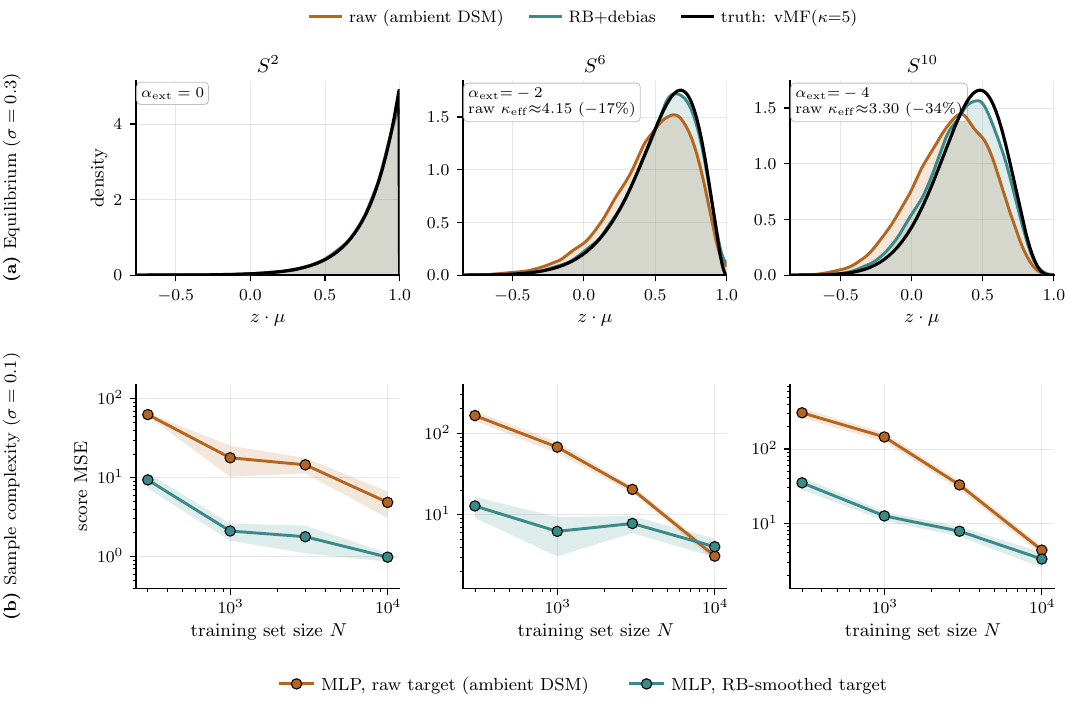}
\caption{\textit{(a)} Densities of $z\cdot\mu$ at the equilibrium of three closed-form Langevin drifts ($\sigma{=}0.3$): the intrinsic score $\nabla_M\log q$ (black), the raw ambient-DSM drift $(1+\sigma^2\alpha_d)\nabla_M\log q$ with $\alpha_d{=}1-d/2$ (orange), and its \cref{cor:sphere-coefficient} debias $(1-\sigma^2\alpha_d)(1+\sigma^2\alpha_d)\nabla_M\log q$ (blue). \textit{(b)} Score MSE of the same MLP architecture and training budget regressing on the raw $T_{\sigma,i}$ target (orange) versus the Rao-Blackwellized $\widehat r_{\sigma,i}$ target (blue). The vertical gap is the sample-efficiency value of Rao-Blackwellized preprocessing and widens with $d$ in the way \cref{thm:variance-collapse} predicts.}
\label{fig:dim-payoff}
\end{figure*}

\Cref{fig:dim-payoff} demonstrates this effect on three hyperspheres, where $g^{\mathrm{inh}}_M\equiv0$ and the scalar debias of \cref{cor:sphere-coefficient} is therefore the complete extrinsic correction.

\section{Discussion}
\label{sec:discussion}

\paragraph{Intrinsic generation and inference.}
\cref{eq:main-intrinsic} identifies the canonical target as an intrinsic score surrogate, $r_\sigma(z)=\nabla_M\log q(z)+\cO(\sigma^2)$, so the intrinsic Langevin dynamics
\begin{equation}
dZ_t = r_\sigma(Z_t)dt + \sqrt{2}dB_t^M
\label{eq:langevin}
\end{equation}
are an $\cO(\sigma^2)$-accurate drift surrogate for Langevin sampling targeting $q$ on $M$. \Cref{thm:variance-collapse} makes this surrogate statistically usable: the raw tangent denoising target has variance diverging at the $d/\sigma^2$ rate, whereas $r_\sigma(\pi(X))$ has bounded variance, so Rao-Blackwellization is not a constant-factor improvement; it removes the singular fiber-noise channel and isolates the intrinsic score signal.

\paragraph{Ambient vs.\ intrinsic DSM.}
The comparison between ambient and intrinsic methods is often framed qualitatively: intrinsic methods require manifold infrastructure while ambient methods do not. \Cref{eq:extrinsic-main} quantifies the exact bias resulting from using ambient instead of intrinsic methods. Given any embedded $M$ and any $\sigma$, the finite-$\sigma$ extrinsic bias of ambient DSM is $\sigma^2\bigl(g^{\rm ext}_M(z)+g^{\rm inh}_M(z)\bigr)$, where $g^{\rm ext}_M=\bigl(\frac12 W_H-\operatorname{Ric}^{\sharp}\bigr)\nabla_M \log q$ and $g^{\rm inh}_M=\frac12\mathcal C_M$, computable from the second fundamental form, its covariant derivative, and the intrinsic score. Practitioners may (i) accept it (if $\sigma$ is small enough that $\sigma^2\norm{g^{\mathrm{ext}}+g^{\mathrm{inh}}}$ is below the target accuracy), (ii) subtract it off, leaving only the intrinsic flat-Tweedie bias $\sigma^2 b_q$, or (iii) avoid regimes where it dominates. The two corrections differ in what they require: $g^{\mathrm{ext}}_M$ is a linear operator applied to an estimate of the score, so its accuracy inherits the score error, whereas $g^{\mathrm{inh}}_M$ is a fixed vector field determined by the embedding alone and can be subtracted without any score estimate.

\paragraph{$S^2$ bias.}
\cref{cor:sphere-coefficient} shows that $S^2$ --- the most commonly used test for manifold-aware score matching --- is a coincidentally benign case: the extrinsic coefficient $1-d/2$ vanishes exactly at $d=2$, so on $S^2$ ambient DSM matches the intrinsic score up to the intrinsic flat-Tweedie bias $\sigma^2 b_q$, without any extrinsic debiasing. $S^2$ is in fact benign for two independent reasons: the Einstein identity $\tfrac12 W_H=\Ric^\sharp$ on unit $S^2$ cancels the score-dependent term, and the round embedding is parallel, so the score-independent term $g^{\mathrm{inh}}_M$ vanishes as well. The formula predicts measurable second-order bias on $S^1$, $S^3$, or $S^d$ for $d\ge 3$ (\cref{fig:dim-payoff}), and on any non-parallel embedding it predicts bias even where the intrinsic score vanishes.

\paragraph{Limitations and open directions.}
Our results are small-noise population statements: we characterize $r_\sigma$ in terms of the true joint law of $(Z,X)$ and its expansion as $\sigma\to 0^+$. A finite-sample local-averaging rate is given in \cref{app:finite-sample}, but the interplay between $\sigma$ and sample size in the regime where the extrinsic bias and the statistical error are of comparable magnitude (i.e.\ how to choose $\sigma$ given $N$) is a natural next question. Our numerical experiments are carried out on $S^d$ and on the flat product torus, whose embeddings are parallel; they therefore isolate $g^{\mathrm{ext}}_M$ and say nothing about $g^{\mathrm{inh}}_M$, which vanishes identically on those supports. Measuring the curvature-inhomogeneity drift requires a non-parallel embedding, which we leave to future work. A sharper practical question is how to correct that drift $g^{\mathrm{inh}}_M=\tfrac12\mathcal C_M$ when the embedding is known only through samples: unlike $g^{\mathrm{ext}}_M$ it does not vanish where the intrinsic score vanishes, so it cannot be neglected in low-density regions, and estimating it requires the third-order local geometry of the embedding rather than just its curvature. Extending the expansion to higher order, to manifolds with boundary, or to anisotropic noise are also open.

Scientific data are often constrained to embedded submanifolds --- e.g. orientation and rotation groups in robotics and protein structure, directional and spherical data in geosciences and astronomy, and products of spaces. The question of whether, and to what order, ambient score models learn the correct intrinsic object is both practically and conceptually central: it informs whether a generic score-matching pipeline may be deployed on manifold-valued data without heavy geometric tools. Our results answer that question quantitatively at small $\sigma$ and identify a canonical Rao-Blackwellized target that is implementation-agnostic, informing the development of score-matching-based pipelines for manifold-constrained data.

\bibliographystyle{icml2026}
\bibliography{references}
\newpage
\appendix
\onecolumn

\section{Some Useful Results from Riemannian Geometry}
\label{app:notation}

Here we define some basic geometric notation and record some useful facts from differential geometry that are used in the main text and in the remaining appendices. We shall follow the convention of \citet{docarmo1992riemannian} and give all statements without proof here.

\paragraph{Submanifold setup.}
Throughout the text $M\subset\R^D$ is taken to be a $C^6$ embedded submanifold without boundary, of dimension $d<D$, with positive reach $\reach(M)>0$. For $z\in M$ we write $T_zM$ and $N_zM$ for its tangent and normal spaces in $\R^D$, and $P_T(z): \R^D\to T_zM$, $P_N(z)\doteq I-P_T(z): \R^D\to N_zM$ for the orthogonal projections. The tangent bundle is $TM=\bigsqcup_{z\in M}T_zM$. The induced Riemannian metric $g_M$ on $M$ is the restriction of the Euclidean inner product on $\R^D$ to each $T_zM$. Gradients on $M$ are denoted $\nabla_M$ and are defined by $\ip{\nabla_M f(z)}{v}=d f(z)[v]$ for $v\in T_zM$; equivalently, $\nabla_M f(z)=P_T(z)\nabla_{\!\R^D} \tilde f(z)$ for any smooth extension $\tilde f$.

\paragraph{Second fundamental form.}
The second fundamental form of $M\subset\R^D$ at $z$ is the symmetric bilinear map \[ \mathrm{II}_z: T_zM\times T_zM \to N_zM, \qquad \mathrm{II}_z(u,v) \doteq P_N(z)\overline{\nabla}_{\!u} \bar v, \] where $\overline{\nabla}$ is the ambient Euclidean connection and $\bar v$ is any tangent extension of $v$. Informally, $\mathrm{II}_z(u,v)$ measures how $M$ bends away from its tangent plane in the $(u,v)$ direction.

\paragraph{Weingarten operator.}
For a normal vector $\nu\in N_zM$, the Weingarten operator $W_\nu: T_zM\to T_zM$ is the symmetric endomorphism characterized by
\begin{equation}
\ip{W_\nu u}{v}_{\R^D} = \ip{\mathrm{II}_z(u,v)}{\nu}_{\R^D},
\qquad u,v\in T_zM.
\label{eq:weingarten-def}
\end{equation}
Equivalently, $W_\nu u = -P_T(z)\overline{\nabla}_{\!u}\tilde\nu$ for any normal extension $\tilde\nu$ of $\nu$. On the unit sphere $S^d\subset\R^{d+1}$ with outward normal $\nu=z$, $W_\nu=-\mathrm{Id}_{T_zM}$.

\paragraph{Mean curvature vector.}
The mean curvature vector at $z$ is \[ H(z) \doteq \sum_{i=1}^d \mathrm{II}_z(e_i,e_i) \in N_zM, \] for any orthonormal basis $(e_i)_{i=1}^d$ of $T_zM$. The operator $W_{H(z)}$, obtained by first forming the mean curvature vector and then taking its Weingarten operator, appears in the extrinsic correction. On $S^d$ with outward normal $\nu=z$, $H(z)=-dz$ and hence $W_{H(z)}=d\mathrm{Id}_{T_zM}$.

\paragraph{Ricci endomorphism.}
The intrinsic Ricci tensor $\Ric$ of $(M,g_M)$ is raised to an endomorphism $\Ric^\sharp_z: T_zM\to T_zM$ via the metric: $\ip{\Ric^\sharp_z u}{v}=\Ric_z(u,v)$. On $S^d$, $\Ric^\sharp_z=(d-1)\mathrm{Id}_{T_zM}$.

\paragraph{Gauss equation.}
In the main text, we link intrinsic and extrinsic curvature via the Gauss equation for a submanifold $M^d\subset\R^D$ with orthonormal normal frame $(n_\alpha)_{\alpha=1}^{D-d}$:
\begin{equation}
\Ric^\sharp_z = W_{H(z)} - \mathcal{S}_z,
\qquad
\mathcal{S}_z \doteq \sum_{\alpha=1}^{D-d} W_{n_\alpha}^2.
\label{eq:gauss-equation}
\end{equation}
This identity lets us rewrite the extrinsic operator $\tfrac12 W_{H(z)} - \Ric^\sharp_z$ equivalently as $\mathcal{S}_z-\tfrac12 W_{H(z)}$; \cref{rem:comments} uses this alternative form.

\paragraph{Covariant derivative of $\mathrm{II}$.}
Let $\nabla^\perp$ denote the normal connection on $NM$. The covariant derivative of the second fundamental form is the $NM$-valued $3$-tensor \[ (\nabla^\perp_w\mathrm{II})_z(u,v) \doteq \nabla^\perp_w\bigl(\mathrm{II}(u,v)\bigr)-\mathrm{II}(\nabla_wu,v)-\mathrm{II}(u,\nabla_wv). \] Because the ambient space is flat, the Codazzi equation states that $(\nabla^\perp_w\mathrm{II})_z(u,v)$ is symmetric in all three arguments. Fixing orthonormal frames $\{e_i\}_{i=1}^d$ of $T_zM$ and $\{n_\alpha\}_{\alpha=1}^{D-d}$ of $N_zM$, we write
\begin{equation}
\mathrm{II}^\alpha_{ij}\doteq\ip{\mathrm{II}_z(e_i,e_j)}{n_\alpha},
\qquad
A^\alpha_{ijk}\doteq\ip{(\nabla^\perp_{e_k}\mathrm{II})_z(e_i,e_j)}{n_\alpha},
\label{eq:codazzi-jet}
\end{equation}
so that $A^\alpha_{ijk}$ is totally symmetric and $\nabla_kH^\alpha=\sum_jA^\alpha_{jjk}$ with $H^\alpha \doteq \sum_i\mathrm{II}^\alpha_{ii}$. We call the embedding \emph{parallel} at $z$ if $\nabla^\perp\mathrm{II}_z=0$, equivalently $A^\alpha_{ijk}=0$. Affine subspaces, round spheres, and Riemannian products of round spheres (e.g.\ the flat $S^1(R_1)\times S^1(R_2)\subset\R^4$) are parallel at every point.

\paragraph{Curvature-inhomogeneity vector.}
The tangent vector
\begin{equation}
\mathcal C_M(z)
\doteq
\nabla_M\norm{\mathrm{II}}^2(z)-\tfrac12\nabla_M\norm{H}^2(z)+\sum_{i=1}^dW_{\nabla^\perp_{e_i}H(z)}e_i
\;\in\;T_zM,
\label{eq:CM-def}
\end{equation}
where $\norm{\mathrm{II}}^2=\sum_{\alpha,i,j}(\mathrm{II}^\alpha_{ij})^2$ and $\norm{H}^2=\sum_\alpha(H^\alpha)^2$, is linear in $\nabla^\perp\mathrm{II}$ and vanishes identically wherever the embedding is parallel. In the frame \cref{eq:codazzi-jet},
\begin{equation}
(\mathcal C_M)_k
=\sum_\alpha\Bigl[\,2\sum_{i,j}\mathrm{II}^\alpha_{ij}A^\alpha_{ijk}
-H^\alpha\sum_jA^\alpha_{jjk}
+\sum_{i,j}\mathrm{II}^\alpha_{ik}A^\alpha_{ijj}\Bigr].
\label{eq:CM-frame}
\end{equation}
Unlike $\Ric^\sharp$, $\mathcal C_M$ is not determined by the intrinsic geometry of $(M,g_M)$: the cylinder over a plane curve of curvature $\kappa$ is intrinsically flat, yet has $\mathcal C_M=2\kappa\kappa'e_1$ with $e_1$ the unit tangent along the curve factor.

\paragraph{Reach and tubular coordinates.}
Positive reach, $\reach(M)>0$, is the largest $r_0$ such that every point $x\in\R^D$ with $\mathrm{dist}(x,M)<r_0$ has a unique nearest point $\pi(x)\in M$. For any $r_0<\reach(M)$, the map \[ \Psi: \{(y,u):y\in M,\ u\in N_yM,\ \norm{u}<r_0\} \to \mathrm{Tub}_{r_0}(M), \qquad \Psi(y,u) \doteq y+u, \] is a $C^3$ diffeomorphism onto the open $r_0$-tube around $M$. Its inverse is $(\pi(x),x-\pi(x))$.

\paragraph{Exponential map.}
The intrinsic exponential map at $z\in M$ is $\Exp_z: T_zM\supseteq B(0,r)\to M$, defined by $\Exp_z(v)=\gamma_v(1)$ where $\gamma_v$ is the geodesic with $\gamma_v(0)=z$ and $\dot\gamma_v(0)=v$. For $r$ smaller than the injectivity radius at $z$, $\Exp_z$ is a $C^3$ diffeomorphism onto its image; normal coordinates on $M$ near $z$ are its inverse.

\section{Proof of the Leading-Order Identification and Variance Collapse}
\label{app:leading-order-proof}
\label{app:variance-collapse-proof}

Here we give the proof of \cref{eq:main-intrinsic} and \cref{thm:variance-collapse}. The argument reduces both statements to a tubular-coordinate Bayes calculation combined with a manifold Stein identity. The argument is self-contained given \cref{prop:fiber-posterior} below, whose proof we defer to \cref{app:fiber-posterior}.

\subsection{Local Coordinates on a Curved Manifold}

Fix $z\in M$. Let \[ F_z(v) \doteq \Exp_z(v), \qquad v\in T_zM, \] be geodesic normal coordinates on a sufficiently small neighborhood of $0 \in T_zM$. Define the tangent chord map
\begin{equation}
G_z(v) \doteq P_T(z)\bigl(F_z(v)-z\bigr)\in T_zM.
\label{eq:Gz}
\end{equation}

\begin{lemma}[Chord expansion]
\label{lem:chord}
Uniformly in $z\in M$, for $v$ sufficiently small,
\begin{equation}
G_z(v)=v+R_{3,z}(v)+R_{\geq 4,z}(v),
\qquad
R_{3,z}(v)=\cO(\norm{v}^3),
\qquad
R_{\geq 4,z}(v)=\cO(\norm{v}^4),
\label{eq:Gz-expansion}
\end{equation}
where the cubic part $R_{3,z}$ is odd: $R_{3,z}(-v)=-R_{3,z}(v)$.
\end{lemma}

\begin{proof}
The ambient Taylor expansion of the exponential map at $z$ is
\[
F_z(v)=z+v+\tfrac12 \mathrm{II}_z(v,v)+C_z(v,v,v)+\cO(\norm{v}^4),
\]
where $\mathrm{II}_z$ is the (symmetric, quadratic) second fundamental form and
$C_z$ is a symmetric trilinear form valued in $\R^D$ coming from the third
Taylor coefficient of $F_z$. Applying $P_T(z)$ annihilates the normal
quadratic term $\tfrac12\mathrm{II}_z(v,v)\in N_zM$. The tangent component of
the cubic term, $R_{3,z}(v)\doteq P_T(z)C_z(v,v,v)$, is odd in $v$
because $C_z$ is symmetric-trilinear and $v\mapsto(v,v,v)$ changes sign
under $v\mapsto -v$. Higher-order corrections are collected in
$R_{\geq 4,z}(v)=\cO(\norm{v}^4)$.
\end{proof}

On the event $\{\pi(X)=z\}$, define \[ V_z \doteq \Exp_z^{-1}(Z)\in T_zM. \] Then \cref{eq:tangent-target-simplified} becomes
\begin{equation}
T_\sigma=\frac{1}{\sigma^2}G_z(V_z)
\qquad\text{on }\{\pi(X)=z\}.
\label{eq:T-local}
\end{equation}

\subsection{Fiber Posterior Normal Form}

\begin{proposition}[Fiber posterior normal form]
\label{prop:fiber-posterior}
There exists $\sigma_0>0$ such that for each $\sigma\in(0,\sigma_0]$ and each
$z\in M$, the conditional law of $V_z$ given $\pi(X)=z$ admits a density of
the form
\begin{equation}
\mu_{\sigma,z}(dv)
=
\frac{1}{\mathcal{Z}_{\sigma,z}}
\gamma_\sigma(v)a_{\sigma,z}(v)dv,
\label{eq:posterior-density}
\end{equation}
where
\[
\gamma_\sigma(v)
 \doteq 
(2\pi \sigma^2)^{-d/2}
\exp\!\left(-\frac{\norm{v}^2}{2\sigma^2}\right)
\]
and
\begin{equation}
a_{\sigma,z}(v)=q(F_z(v))J_z(v)\Lambda_{\sigma,z}(v).
\label{eq:a-factor}
\end{equation}
Here $J_z$ is the geodesic-coordinate volume Jacobian and $\Lambda_{\sigma,z}$
is a smooth positive correction factor satisfying, uniformly in $z$,
\begin{align}
J_z(v)&=1+\cO(\norm{v}^2), & \nabla_v J_z(0)&=0, \label{eq:Jz-expansion}\\
\Lambda_{\sigma,z}(v)&=1+\cO(\sigma^2+\norm{v}^2), & \nabla_v \Lambda_{\sigma,z}(0)&=0. \label{eq:Lambdaz-expansion}
\end{align}
Moreover, for every integer $k\ge 1$,
\begin{equation}
\sup_{z\in M}\int \norm{v}^k\mu_{\sigma,z}(dv)\le C_k \sigma^k.
\label{eq:moment-estimate}
\end{equation}
\end{proposition}

\begin{proof}
The proof is rather involved and deserves its own appendix, see \cref{app:fiber-posterior}. The main point is that conditioning on $\pi(X)=z$
produces a posterior over latent tangent displacements $v$ whose leading
factor is the centered Gaussian $\exp(-\norm{v}^2/(2\sigma^2))$, while all
geometric corrections are smooth and even to first order.
\end{proof}

\subsection{A Manifold Stein Identity}

\begin{lemma}[Posterior Stein identity]
\label{lem:stein}
For every $z\in M$,
\begin{equation}
\frac{1}{\sigma^2}\E_{\mu_{\sigma,z}}[v]
=
\E_{\mu_{\sigma,z}}\!\left[\nabla_v \log a_{\sigma,z}(v)\right].
\label{eq:stein}
\end{equation}
\end{lemma}

\begin{proof}
Fix an orthonormal basis in $T_zM\cong \R^d$. For coordinate $i$,
\[
0=\int \partial_i\!\bigl(\gamma_\sigma(v)a_{\sigma,z}(v)\bigr)dv
\]
because $\gamma_\sigma(v)a_{\sigma,z}(v)$ decays rapidly. Using
\[
\partial_i \gamma_\sigma(v)=-\frac{v_i}{\sigma^2}\gamma_\sigma(v),
\]
we obtain
\[
0
=
\int
\left(
\partial_i a_{\sigma,z}(v)-\frac{v_i}{\sigma^2}a_{\sigma,z}(v)
\right)\gamma_\sigma(v)dv.
\]
Divide by the normalizing constant $\mathcal{Z}_{\sigma,z}$ to get
\[
\frac{1}{\sigma^2}\E[v_i]
=
\E[\partial_i \log a_{\sigma,z}(v)].
\]
Collecting coordinates gives \cref{eq:stein}.
\end{proof}

\subsection{Proof of the Leading-Order Expansion \cref{eq:main-intrinsic}}
\label{app:proof-intrinsic}

\begin{proof}[Proof of \cref{eq:main-intrinsic}]
Fix $z\in M$. From \cref{eq:T-local},
\[
r_\sigma(z)
=
\frac{1}{\sigma^2}\E_{\mu_{\sigma,z}}[G_z(v)].
\]
Using \cref{lem:chord}, $G_z(v)=v+R_z(v)$ with
$R_z(v)\doteq R_{3,z}(v)+R_{\geq 4,z}(v)$, where
$R_{3,z}(v)=\cO(\norm{v}^3)$ is odd in $v$ and
$R_{\geq 4,z}(v)=\cO(\norm{v}^4)$ is a higher-order remainder, we obtain
\begin{equation}
r_\sigma(z)
=
\frac{1}{\sigma^2}\E[v]
+
\frac{1}{\sigma^2}\E[R_z(v)].
\label{eq:r-split}
\end{equation}

By \cref{lem:stein},
\[
\frac{1}{\sigma^2}\E[v]
=
\E[\nabla_v \log a_{\sigma,z}(v)].
\]
Taylor-expand $\nabla_v \log a_{\sigma,z}(v)$ around $0$:
\[
\nabla_v \log a_{\sigma,z}(v)
=
\nabla_v \log a_{\sigma,z}(0)
+
B_{\sigma,z}v
+
\cO(\norm{v}^2),
\]
where $B_{\sigma,z}$ is uniformly bounded in $z$ and $\sigma$ for sufficiently
small $\sigma$. Taking expectation and using \cref{eq:moment-estimate},
\[
\frac{1}{\sigma^2}\E[v]
=
\nabla_v \log a_{\sigma,z}(0)
+
B_{\sigma,z}\E[v]
+
\cO(\E\norm{v}^2).
\]
Since $\E\norm{v}^2=\cO(\sigma^2)$, the left-hand side is bounded, so
$\E[v]=\cO(\sigma^2)$ and therefore
\begin{equation}
\frac{1}{\sigma^2}\E[v]
=
\nabla_v \log a_{\sigma,z}(0)
+
\cO(\sigma^2).
\label{eq:v-leading}
\end{equation}

Now use \cref{eq:a-factor}:
\[
\nabla_v \log a_{\sigma,z}(0)
=
\nabla_v \log q(F_z(v))\big|_{v=0}
+
\nabla_v \log J_z(0)
+
\nabla_v \log \Lambda_{\sigma,z}(0).
\]
By construction,
\[
\nabla_v \log q(F_z(v))\big|_{v=0}
=
\nabla_M \log q(z).
\]
By \cref{eq:Jz-expansion} and \cref{eq:Lambdaz-expansion},
\[
\nabla_v \log J_z(0)=0,
\qquad
\nabla_v \log \Lambda_{\sigma,z}(0)=0.
\]
Hence
\begin{equation}
\frac{1}{\sigma^2}\E[v]
=
\nabla_M \log q(z)+\cO(\sigma^2).
\label{eq:v-main-final}
\end{equation}

The cubic part $R_{3,z}$ has zero Gaussian expectation by oddness; under
the posterior density \cref{eq:posterior-density}, the first nonzero
contribution from $R_{3,z}$ comes from the odd linear perturbation of
$a_{\sigma,z}(v)$, which is order $\norm{v}$. The quartic remainder
$R_{\geq 4,z}(v)=\cO(\norm{v}^4)$ contributes $\cO(\sigma^4)$ directly by
moment bounds. Combining with $R_{3,z}(v)=\cO(\norm{v}^3)$ and Gaussian
moments yields
\[
\E[R_z(v)] = \cO(\sigma^4),
\]
uniformly in $z$ (a more explicit estimate is given in
\cref{app:chord}). Therefore
\begin{equation}
\frac{1}{\sigma^2}\E[R_z(v)] = \cO(\sigma^2).
\label{eq:chord-remainder}
\end{equation}

Substituting \cref{eq:v-main-final} and \cref{eq:chord-remainder} into
\cref{eq:r-split} gives
\[
r_\sigma(z)=\nabla_M \log q(z)+\cO(\sigma^2),
\]
uniformly in $z$.
\end{proof}

\subsection{Proof of \cref{thm:variance-collapse}}
\label{app:proof-variance-collapse}

We are finally ready to prove \cref{thm:variance-collapse}.

\begin{proof}[Proof of \cref{thm:variance-collapse}]
Fix $z\in M$. By \cref{eq:T-local},
\[
T_\sigma=\frac{1}{\sigma^2}G_z(V_z)
\qquad\text{conditionally on }\pi(X)=z.
\]
Hence
\[
\Var(T_\sigma\mid \pi(X)=z)
=
\frac{1}{\sigma^4}\Var_{\mu_{\sigma,z}}(G_z(v)).
\]

By \cref{lem:chord},
\[
\norm{G_z(v)}^2
=
\norm{v}^2+\cO(\norm{v}^4).
\]
Using \cref{eq:moment-estimate},
\[
\E \norm{v}^2 = d\sigma^2+\cO(\sigma^4),
\qquad
\E \norm{v}^4 = \cO(\sigma^4),
\]
so
\begin{equation}
\E \norm{G_z(v)}^2 = d\sigma^2+\cO(\sigma^4).
\label{eq:G-second-moment}
\end{equation}

On the other hand,
\[
\E[G_z(v)]
=
\sigma^2 r_\sigma(z)
=
\sigma^2 \nabla_M \log q(z)+\cO(\sigma^4)
\]
by \cref{eq:main-intrinsic}, hence
\begin{equation}
\norm{\E[G_z(v)]}^2 = \cO(\sigma^4).
\label{eq:G-mean-square}
\end{equation}
Subtracting \cref{eq:G-mean-square} from \cref{eq:G-second-moment},
\[
\Var_{\mu_{\sigma,z}}(G_z(v))
=
d\sigma^2+\cO(\sigma^4).
\]
Dividing by $\sigma^4$ yields \cref{eq:cond-var-blowup}.

Finally, the law of total variance applied to $T_\sigma$ with conditioning on $\pi(X)$ gives
\[
\Var(T_\sigma)
=
\Var(r_\sigma(\pi(X)))
+
\E[\Var(T_\sigma\mid \pi(X))].
\]
Averaging \cref{eq:cond-var-blowup} gives \cref{eq:uncond-var-blowup}.
Since \cref{eq:main-intrinsic} implies
\[
r_\sigma(\pi(X))
=
\nabla_M \log q(\pi(X))+\cO_{L^2}(\sigma^2),
\]
the variance of the Rao-Blackwellized target remains $\cO(1)$.
\end{proof}

\section{Derivation of the Fiber Posterior Normal Form}
\label{app:fiber-posterior}

Here we establish \cref{prop:fiber-posterior}. The proof proceeds by a sequence of exact changes of variables.

Let us first describe our approach. We shall use that the Federer-Gray tube map provides a global $C^3$ diffeomorphism between a neighborhood of $M$ in $\R^D$ and the normal bundle, with an explicit Jacobian controlled by the Weingarten operator. Normal coordinates on $M$ centered at $z$ decompose the chord $F_z(v)-z$ into its tangential and normal components, with the tangential part agreeing with $v$ up to a cubic odd remainder and the normal part given to leading order by the second fundamental form. Using these two changes of variables, one obtains an explicit expression for the joint density of $(Z,X)$ in tubular-normal coordinates. Then, integrating out the ambient normal coordinate of $X$ yields the conditional law of $V_z=\Exp_z^{-1}(Z)$ given $\pi(X)=z$, which has a $(D-d)$-dimensional Gaussian integral $I_{\sigma,z}(v)$ against the tube Jacobian. The substitution trick that defines $\Lambda_{\sigma,z}(v)$ absorbs the tangential distortion of the chord into a single multiplicative correction factor whose deviation from one is quartic in $v$, not cubic; this is where the normal-form claim \cref{eq:posterior-density} becomes explicit. Finally, polynomial moment bounds on $\mu_{\sigma,z}$ follow from compactness of $M$, strict positivity of $q$, and Gaussian tail estimates, and the same tail estimate shows that the event $X\notin\Tub_{r_0}(M)$ contributes only $\cO(\exp(-c/\sigma^2))$ and hence does not affect any of the polynomial-in-$\sigma$ expansions used elsewhere.

Throughout this section, fix $z\in M$, and recall $F_z(v)=\Exp_z(v)$. Constants $C_k,C_k'$ depend only on $M$, $\reach(M)$, $\norm{q}_{C^4(M)}$, $\inf_M q$, and ambient bounds on the second fundamental form and its covariant derivative; in particular not on $z$ or $\sigma$.

\begin{proof}[Proof of \cref{prop:fiber-posterior}]

Since $r_0<\reach(M)$, the Federer tube map \[ \Psi:\{(y,u):y\in M,\ u\in N_yM,\ \norm{u}<r_0\}\to \Tub_{r_0}(M), \qquad \Psi(y,u) \doteq y+u \] is a $C^3$ diffeomorphism onto its image, with inverse $(\pi(x),x-\pi(x))$. Choose an orthonormal frame $(e_1,\ldots,e_d)$ of $TM$ on a neighborhood $U\ni z$ and complete it to an ambient orthonormal frame $(e_1,\ldots,e_d,n_{d+1},\ldots,n_D)$ smoothly in $y\in U$, with $n_\alpha(y)\in N_yM$. Writing $u=\sum_\alpha u^\alpha n_\alpha(y)$, the Federer-Gray tube formula on positive reach yields
\begin{equation}
 dx
 =
 \det\!\bigl(I_d - W_u(y)\bigr)d\vol_M(y)du,
 \label{eq:tube-jacobian}
\end{equation}
where $W_u(y):T_yM\to T_yM$ is the Weingarten operator in the normal direction $u$, so $W_u(y)$ is linear in $u$ and $\tr(W_u(y))=\ip{H(y)}{u}$ with $H(y)\in N_yM$ the mean curvature vector. In particular,
\begin{equation}
 J^\mathrm{tub}_z(u)
 \doteq \det\!\bigl(I_d-W_u(z)\bigr)
 =1-\ip{H(z)}{u}+\mathcal{Q}_z(u),
 \label{eq:tube-jacobian-expansion}
\end{equation}
where $\mathcal{Q}_z(u)=\cO(\norm{u}^2)$ uniformly in $z$, and $J^\mathrm{tub}_z(u)$ is bounded above and below by positive constants on $\norm{u}<r_0$.

Passing to normal coordinates on $M$ at $z$, write $y=F_z(v)$. The Riemannian volume satisfies
\begin{equation}
 d\vol_M(F_z(v))=J_z(v)dv,
 \qquad
 J_z(v)=1-\tfrac{1}{6}\ip{\Ric_zv}{v}+\cO(\norm{v}^3),
 \label{eq:intrinsic-jac}
\end{equation}
so in particular $J_z(0)=1$, $\nabla_v J_z(0)=0$, and $J_z(v)=1+\cO(\norm{v}^2)$ uniformly in $z$. Decomposing the chord along the tangent and normal spaces at $z$,
\begin{equation}
 F_z(v)-z
 =
 G_z(v)+N_z(v),
 \quad
 G_z(v) \doteq P_T(z)(F_z(v)-z),
 \quad
 N_z(v) \doteq P_N(z)(F_z(v)-z),
 \label{eq:tn-decomposition}
\end{equation}
\cref{lem:chord} gives $G_z(v)=v+R_{3,z}(v)+R_{\geq 4,z}(v)$ with $R_{3,z}(v)=\cO(\norm{v}^3)$, $R_{\geq 4,z}(v)=\cO(\norm{v}^4)$, and the cubic part odd: $R_{3,z}(-v)=-R_{3,z}(v)$. The Gauss lemma and the standard expansion of the exponential map yield
\begin{equation}
 N_z(v)
 =
 \tfrac{1}{2}\mathrm{II}_z(v,v)+\cO(\norm{v}^3),
 \label{eq:normal-part}
\end{equation}
where $\mathrm{II}_z(v,v)\in N_zM$ is the second fundamental form, so $N_z(v)$ is even in $v$ to leading order and purely quadratic.

The joint law of $(Z,X)$ has density, with respect to $d\vol_M(Z)\otimes dX$, \[ f_{Z,X}(y,x) = q(y)\cdot(2\pi\sigma^2)^{-D/2}\exp\!\left(-\frac{\norm{x-y}^2}{2\sigma^2}\right). \] Changing variables in $x$ via $\Psi$ and using \cref{eq:tube-jacobian},
\begin{equation}
 f_{Z,\pi(X),\perp}(y,z^*,u)
 =
 q(y)\cdot(2\pi\sigma^2)^{-D/2}
 \exp\!\left(-\frac{\norm{z^*+u-y}^2}{2\sigma^2}\right)
 \det\!\bigl(I_d-W_u(z^*)\bigr)
 \label{eq:joint-density}
\end{equation}
with respect to $d\vol_M(y)\otimes d\vol_M(z^*)\otimes du$, where $(z^*,u)\in\{(z',u'):z'\in M,\ u'\in N_{z'}M,\ \norm{u'}<r_0\}$.

To obtain the conditional law $\mu_{\sigma,z}(dv)$ of $V_z=\Exp_z^{-1}(Z)$ given $\pi(X)=z$, parameterize $y=F_z(v)$ and integrate out the normal coordinate $u$ of $X$ in \cref{eq:joint-density} evaluated at $z^*=z$:
\begin{equation}
 \mathcal{Z}_{\sigma,z}\mu_{\sigma,z}(dv)
 =
 q(F_z(v))J_z(v)\cdot k_{\sigma,z}(v)dv,
 \label{eq:posterior-unnormalized}
\end{equation}
where
\begin{align*}
 k_{\sigma,z}(v)
 & \doteq (2\pi\sigma^2)^{-D/2}
 \int_{N_zM}
 \exp\!\left(-\frac{\norm{F_z(v)-z-u}^2}{2\sigma^2}\right)
 J^\mathrm{tub}_z(u)du,
\end{align*}
and $\mathcal{Z}_{\sigma,z}$ is the marginal density of $\pi(X)$ at $z$. By \cref{eq:tn-decomposition} and the orthogonality $T_zM\perp N_zM$, \[ \norm{F_z(v)-z-u}^2 = \norm{G_z(v)}^2+\norm{N_z(v)-u}^2, \] so the tangential and normal components of the quadratic form decouple and
\begin{equation}
 k_{\sigma,z}(v)
 =
 (2\pi\sigma^2)^{-d/2}
 \exp\!\left(-\frac{\norm{G_z(v)}^2}{2\sigma^2}\right)
 \cdot
 I_{\sigma,z}(v),
 \label{eq:k-factor}
\end{equation}
where
\begin{equation}
 I_{\sigma,z}(v)
 \doteq
 \int_{N_zM}\phi_\sigma^{(D-d)}(u-N_z(v))J^\mathrm{tub}_z(u)du,
 \label{eq:I-gaussian-integral}
\end{equation}
and $\phi_\sigma^{(D-d)}$ is the isotropic Gaussian density on $N_zM\simeq\R^{D-d}$.

The tube integral \cref{eq:I-gaussian-integral} is an ambient Gaussian expectation with mean $N_z(v)$ and isotropic variance $\sigma^2 I_{D-d}$. By \cref{eq:tube-jacobian-expansion} and Gaussian-moment calculations,
\begin{align*}
 I_{\sigma,z}(v)
 &=
 1-\ip{H(z)}{\E[u]}+\E[\mathcal{Q}_z(u)]
 \\
 &=
 1-\ip{H(z)}{N_z(v)}
 +\mathcal{Q}_z(N_z(v))
 +\sigma^2\cdot \tfrac{1}{2}\tr(D^2\mathcal{Q}_z(0))
 +\cO(\sigma^4+\sigma^2\norm{N_z(v)}^2).
\end{align*}
By \cref{eq:normal-part}, $\ip{H(z)}{N_z(v)}=\cO(\norm{v}^2)$, and every term above is $1+\cO(\sigma^2+\norm{v}^2)$ with vanishing first-order part in $v$. The tangential part of the chord is not exactly $v$, however, and the mismatch must be absorbed to expose the Gaussian density $\gamma_\sigma$. Define
\begin{equation}
 \Lambda_{\sigma,z}(v)
 \doteq
 \exp\!\left(\frac{\norm{v}^2-\norm{G_z(v)}^2}{2\sigma^2}\right)
 \cdot
 I_{\sigma,z}(v).
 \label{eq:Lambda-def}
\end{equation}
Writing $R_z(v)\doteq R_{3,z}(v)+R_{\geq 4,z}(v)=\cO(\norm{v}^3)$ so that $G_z(v)=v+R_z(v)$, we have $\norm{G_z(v)}^2=\norm{v}^2+2\ip{v}{R_z(v)}+\norm{R_z(v)}^2$, whence the exponent in \cref{eq:Lambda-def} is $\cO(\norm{v}^4/\sigma^2)$. Since we will apply this only on the range $\norm{v}\lesssim\sigma\sqrt{\log(1/\sigma)}$ (the Gaussian typical set), this exponent is $\cO(\sigma^2\log^2(1/\sigma))$, and in particular \[ \Lambda_{\sigma,z}(v)=1+\cO(\sigma^2+\norm{v}^2), \qquad \nabla_v\Lambda_{\sigma,z}(0)=0, \] uniformly in $z$, as claimed in \cref{eq:Lambdaz-expansion}. Substituting \cref{eq:k-factor} and \cref{eq:Lambda-def} into \cref{eq:posterior-unnormalized},
\begin{equation}
 \mu_{\sigma,z}(dv)
 =
 \frac{1}{\mathcal{Z}'_{\sigma,z}}
 \gamma_\sigma(v)q(F_z(v))J_z(v)\Lambda_{\sigma,z}(v)dv,
 \label{eq:posterior-normalform}
\end{equation}
with $\mathcal{Z}'_{\sigma,z}$ the normalizing constant of \cref{eq:posterior-normalform}. This is exactly \cref{eq:posterior-density}, proving the normal-form claim.

For the moment bounds \cref{eq:moment-estimate}, write $a_{\sigma,z}(v)=q(F_z(v))J_z(v)\Lambda_{\sigma,z}(v)$. By compactness of $M$, strict positivity of $q$, and the expansions above, there exist constants $c_0,C_0>0$ and $\sigma_0>0$ such that $c_0\le a_{\sigma,z}(v)\le C_0$ for all $z\in M$, $\sigma\in(0,\sigma_0]$, and all $\norm{v}<r_0/2$. Outside $\{\norm{v}<r_0/2\}$, the Gaussian factor $\gamma_\sigma(v)$ decays faster than any polynomial in $\sigma$, contributing only $\cO(\exp(-c/\sigma^2))$ to any moment. Hence, for every integer $k\ge 1$, \[ \int \norm{v}^k\mu_{\sigma,z}(dv) \le \frac{C_0}{c_0}\int \norm{v}^k\gamma_\sigma(v)dv +\cO(e^{-c/\sigma^2}) \le C_k\sigma^k, \] proving \cref{eq:moment-estimate} and completing the proof of \cref{prop:fiber-posterior}.

\end{proof}

\begin{remark}[Restriction to $\mathcal{E}_\sigma$]
The identities above are derived on the event $\mathcal{E}_\sigma\doteq\{X\in
\Tub_{r_0}(M)\}$. Under the Gaussian corruption model with $Z\in M$
compact, the complement $\mathcal{E}_\sigma^c$ has probability
$\cO(\exp(-c/\sigma^2))$ for some $c>0$ depending only on $r_0$. Since
every moment of $T_\sigma$ is polynomial in $\sigma^{-1}$, the contribution
of $\mathcal{E}_\sigma^c$ is negligible in every expansion of the form
$\cO(\sigma^k)$ used in
\crefrange{sec:setup}{sec:leading-order}.
\end{remark}

\section{A More Explicit Estimate for the Chord Correction}
\label{app:chord}

\begin{lemma}[Chord Remainder Moment]
\label{lem:chord-moment}
Let $R_z(v)\doteq R_{3,z}(v)+R_{\geq 4,z}(v)=G_z(v)-v$ be the chord
remainder of \cref{lem:chord}, with $R_{3,z}$ the odd cubic part and
$R_{\geq 4,z}(v)=\cO(\norm{v}^4)$, and let $\mu_{\sigma,z}$ be the fiber
posterior of \cref{prop:fiber-posterior}. Then, uniformly in $z\in M$,
\[
\E_{\mu_{\sigma,z}}[R_z(v)]=\cO(\sigma^4).
\]
\end{lemma}

\begin{proof}
Expand the posterior density as
\[
a_{\sigma,z}(v)=a_{\sigma,z}(0)\bigl(1+\ell_z(v)+m_{\sigma,z}(v)\bigr),
\qquad
\ell_z(v)=\ip{b_z}{v},
\quad
m_{\sigma,z}(v)=\cO(\norm{v}^2+\sigma^2),
\]
and split the expectation into cubic and higher-order pieces.

\emph{Cubic part.} Since $R_{3,z}(v)=\cO(\norm{v}^3)$ is odd in $v$,
\[
\int R_{3,z}(v)\gamma_\sigma(v)dv=0
\quad\text{and}\quad
\int R_{3,z}(v)\,(\text{even part of }m_{\sigma,z})\,\gamma_\sigma(v)dv=0.
\]
The first nonzero contribution comes from the linear piece,
$\int R_{3,z}(v)\ell_z(v)\gamma_\sigma(v)dv$, which is $\cO(\sigma^4)$ by
Gaussian moment bounds; all remaining cubic-part contributions are of
order at least $\sigma^5$.

\emph{Higher-order part.} Since $R_{\geq 4,z}(v)=\cO(\norm{v}^4)$, a direct
moment bound gives $\E_{\mu_{\sigma,z}}[R_{\geq 4,z}(v)]=\cO(\sigma^4)$
without needing any cancellation.

Adding the two pieces yields the stated $\cO(\sigma^4)$ bound.
\end{proof}

\section{Alternative Intrinsic Target Based on the Logarithmic Map}
\label{app:logmap}

An intrinsically cleaner target is \[ \widetilde T_\sigma \doteq \frac{1}{\sigma^2}\Exp^{-1}_{\pi(X)}(Z)\in T_{\pi(X)}M. \] In local coordinates, $\widetilde T_\sigma=V_z/\sigma^2$ on $\{\pi(X)=z\}$. The following $L^2$ estimate shows that $T_\sigma$ and $\widetilde T_\sigma$ agree at leading order.

\begin{proposition}[Equivalence of Ambient and Logmap Targets]
\label{prop:logmap-equivalence}
Under the assumptions of \cref{sec:setup},
\[
T_\sigma-\widetilde T_\sigma=\cO_{L^2}(\sigma).
\]
\end{proposition}

\begin{proof}
By \cref{lem:chord}, $T_\sigma-\widetilde T_\sigma=R_z(V_z)/\sigma^2$. Since
$R_z(V_z)=\cO(\norm{V_z}^3)$ and $\E\norm{V_z}^2=\cO(\sigma^2)$,
\[
\E\norm{T_\sigma-\widetilde T_\sigma}^2
=
\frac{1}{\sigma^4}\E\cO(\norm{V_z}^6)
=
\cO(\sigma^2).
\]
\end{proof}

We chose $T_\sigma$ in the main text because it arises directly from ambient denoising score matching.

\section{Second-Order Refinement}
\label{app:second-order}

The proof of \cref{eq:main-intrinsic} shows that every first-order geometric correction vanishes by symmetry. In this appendix we compute the $\sigma^2$ coefficient exactly in the flat case and state the structural form of the curved-case expansion. The flat computation is a one-line consequence of the exact Tweedie identity \cref{eq:flat-rb}; the curved statement follows the same strategy as \cref{eq:main-intrinsic} with two additional orders of Taylor expansion: the quadratic jet of the geometric weight produces the curvature operator, and its cubic jet produces the additive drift.

\subsection{Exact Flat-Case Coefficient}

\begin{proposition}[Flat-case $\sigma^2$ expansion]
\label{prop:flat-second-order}
Assume the flat setting of~\cref{sec:flat}. Assume further that
$q\in C^5(V)$ is strictly positive with $\norm{\nabla_V^k \log q}_\infty<\infty$
for $k\le 5$. Then, uniformly on compact subsets of $V$,
\begin{equation}
 r_\sigma(t)
 =
 \nabla_V \log q(t)
 +
 \frac{\sigma^2}{2}\nabla_V\!\left[\Delta_V \log q + \norm{\nabla_V \log q}^2\right](t)
 +
 \cO(\sigma^4),
 \qquad \sigma\to 0^+.
 \label{eq:flat-second-order}
\end{equation}
\end{proposition}

\begin{proof}
By \cref{eq:flat-rb}, $r_\sigma(t)=\nabla_V\log p_T(t)$ where $p_T=q*\phi_\sigma^{(d)}$. A fourth-order Taylor expansion of $q$ inside the convolution gives
\[
p_T(t)
=
\E_{w\sim\mathcal{N}(0,\sigma^2 I_d)}[q(t-w)]
=
q(t)+\frac{\sigma^2}{2}\Delta_V q(t)+\frac{\sigma^4}{8}\Delta_V^2 q(t)+\cO(\sigma^6).
\]
Dividing by $q(t)>0$,
\[
\frac{p_T(t)}{q(t)}
=
1+\frac{\sigma^2}{2}\frac{\Delta_V q(t)}{q(t)}+\cO(\sigma^4).
\]
Using that $\Delta_V q/q=\Delta_V\log q+\norm{\nabla_V \log q}^2$,
\[
\log p_T(t)-\log q(t)
=
\frac{\sigma^2}{2}\!\left[\Delta_V\log q(t)+\norm{\nabla_V \log q(t)}^2\right]+\cO(\sigma^4).
\]
Taking $\nabla_V$ of both sides yields \cref{eq:flat-second-order}. The
fifth-derivative bound on $\log q$ controls the gradient of the
$\cO(\sigma^4)$ remainder, giving uniform $\cO(\sigma^4)$ control of the
score remainder on compact sets.
\end{proof}

The Tweedie identity $r_\sigma(t)=\nabla_V\log p_T(t)$ is exact; \cref{eq:flat-second-order} is the corresponding asymptotic expansion of its $\sigma\to 0^+$ behavior, with the Gaussian-smoothing bias of $q$ entering as the explicit $\sigma^2$ correction of the score.

\subsection{Structural Form in the Curved Case}

For curved $M$, the expansion \cref{eq:flat-second-order} picks up three additional geometric contributions: an intrinsic Riemannian smoothing bias depending on the Ricci tensor of $(M,g_M)$, an extrinsic curvature correction depending on the second fundamental form of the embedding, and an additive drift depending on the covariant derivative of the second fundamental form. The last of these is independent of $q$ and is generated by the cubic jet of the embedding.

\begin{proposition}[Curved-case $\sigma^2$ expansion]
\label{prop:curved-second-order}
Under the assumptions of \cref{sec:setup}, uniformly in $z\in M$,
\begin{equation}
 r_\sigma(z)
 =
 \nabla_M\log q(z)
 +
 \sigma^2\bigl[b_q(z)+g_M^{\mathrm{ext}}(z)+g_M^{\mathrm{inh}}(z)\bigr]
 +
 \cO(\sigma^4),
 \qquad \sigma\to 0^+,
 \label{eq:curved-second-order}
\end{equation}
where the intrinsic (flat-Tweedie) term is
\begin{equation}
 b_q(z)
 =
 \tfrac{1}{2}\nabla_M\!\left[\Delta_M \log q + \norm{\nabla_M \log q}^2\right](z),
 \label{eq:bq-flat}
\end{equation}
and the extrinsic term is
\begin{equation}
 g_M^{\mathrm{ext}}(z)
 =
 \Bigl(\tfrac{1}{2}W_{H(z)}-\Ric^\sharp_z\Bigr)\!\bigl(\nabla_M \log q(z)\bigr)
 =
 \Bigl(\mathcal S_z-\tfrac{1}{2}W_{H(z)}\Bigr)\!\bigl(\nabla_M \log q(z)\bigr),
 \label{eq:gM-ext}
\end{equation}
and the curvature-inhomogeneity term is
\begin{equation}
 g_M^{\mathrm{inh}}(z)
 =
 \tfrac{1}{2}\mathcal C_M(z),
 \label{eq:gM-inh}
\end{equation}
with $W_u$ the Weingarten operator in normal direction $u$,
$H(z)=\sum_{i=1}^d \mathrm{II}_z(e_i,e_i)\in N_zM$ the mean curvature vector,
$\Ric^\sharp_z:T_zM\to T_zM$ the Ricci endomorphism,
$\mathcal S_z \doteq \sum_\alpha W_{n_\alpha}^2$ for any orthonormal normal frame
$\{n_\alpha\}$, and $\mathcal C_M(z)$ the curvature-inhomogeneity vector
\cref{eq:CM-def}. The two expressions in \cref{eq:gM-ext} are equivalent via
the Gauss equation $\Ric^\sharp_z=W_{H(z)}-\mathcal S_z$; $g_M^{\mathrm{inh}}$
is linear in $\nabla^\perp\mathrm{II}_z$ and vanishes whenever the embedding is
parallel at $z$. In the flat case
$M=V$ one has $\mathrm{II}\equiv 0$, so $g_M^{\mathrm{ext}}\equiv g_M^{\mathrm{inh}}\equiv 0$ and
\cref{eq:curved-second-order} reduces to
\cref{eq:flat-second-order}.
\end{proposition}

\begin{proof}[Proof outline; a full coordinate derivation and uniformity bounds are given in \cref{app:general-extrinsic}]
We work in graph coordinates at $z$: choose orthonormal frames
$\{e_i\}$ of $T_zM$ and $\{n_\alpha\}$ of $N_zM$, and parametrize a
neighborhood of $z$ in $M$ by
$s\mapsto y(s) \doteq z+s+h(s)\in\R^D$, where $s\in T_zM$ is tangential and
$h(s)\in N_zM$ satisfies $h(0)=0$, $dh(0)=0$,
$\partial_i\partial_j h^\alpha(0)=\mathrm{II}^\alpha_{ij}(z)$. A direct
computation from
$g_{ij}(s)=\delta_{ij}+\sum_\alpha \partial_i h^\alpha(s)\partial_j h^\alpha(s)$
and the Taylor expansion
$\partial_i h^\alpha(s)=\mathrm{II}^\alpha_{ik}(z)s^k+\cO(\norm{s}^2)$
gives the induced graph-Jacobian expansion
\begin{equation}
 J_{\mathrm{gr}}(s) \doteq \sqrt{\det g(s)},
 \qquad
 J_{\mathrm{gr}}(s)=1+\tfrac12\ip{\mathcal S_zs}{s}+\cO(\norm{s}^3),
 \qquad
 \mathcal S_z=\sum_\alpha W_{n_\alpha}^2.
 \label{eq:graph-jac}
\end{equation}
This graph-coord expansion differs from the classical
$\log\sqrt{\det g}=-\tfrac16\Ric(s,s)+\cO$ of geodesic normal
coordinates; graph coordinates use the ambient tangent offset $s$, not
arclength, and the quadratic term is the Gram form
$\sum_\alpha\mathrm{II}^\alpha(s,\cdot)\mathrm{II}^\alpha(s,\cdot)$, i.e.\
$\ip{\mathcal S s}{s}$, rather than the Ricci form. The chord projection
is trivial: $P_{T_zM}(y(s)-z)=s$ by construction, so the tangent component
of $r_\sigma(z)$ equals the posterior mean of $s/\sigma^2$ exactly,
without any separate chord correction.

The second ingredient is the fiber factor. Given $\pi(X)=z$, write
$X=z+\sum_\alpha u_\alpha n_\alpha$ with $u\in N_zM\cong\R^{D-d}$; the
ambient volume element in tubular coordinates contributes the tube
Jacobian $\det(I-\sum_\alpha u_\alpha W_{n_\alpha})$, and the ambient
Gaussian factors as
$e^{-\norm{s}^2/(2\sigma^2)}e^{-\norm{u-h(s)}^2/(2\sigma^2)}$. Integrating
$u$ out gives the fiber factor
\[
 F_\sigma(s)
 \doteq \int_{N_zM}\phi_{\sigma^2 I}(u-h(s))\det\bigl(I-\textstyle\sum_\alpha u_\alpha W_{n_\alpha}\bigr)du.
\]
Substitute $u=h(s)+\eta$ and expand $\det(I-W_{h(s)+\eta})$ as a
polynomial in $u$. Gaussian averaging in $\eta$ kills all odd moments, so
no term of the form $\sigma^2\cdot(\text{linear in }s)$ survives; the
remaining $s$-independent moments absorb into a constant
$C_\sigma=1+\cO(\sigma^2)$. The leading $s$-dependent piece is
$-\ip{H(z)}{h(s)}$, and since
$h(s)=\tfrac12\sum_\alpha\ip{\mathrm{II}^\alpha s}{s}n_\alpha+\cO(\norm{s}^3)$
with $\ip{H(z)}{\mathrm{II}(s,s)}=\ip{W_{H(z)}s}{s}$, we obtain
\begin{equation}
 F_\sigma(s)=C_\sigma\bigl[1-\tfrac12\ip{W_{H(z)}s}{s}+\cO(\norm{s}^3)+\cO(\sigma^2\norm{s}^2)+\cO(\sigma^4)\bigr].
 \label{eq:fiber-factor-body}
\end{equation}
\Cref{eq:graph-jac,eq:fiber-factor-body} are truncations at quadratic order; the
same two computations carried one Taylor order further
(\cref{lem:graph-jacobian,lem:fiber-factor}) produce a cubic term that the
quadratic truncation cannot see. Combining them and using
the Gauss equation $\Ric^\sharp_z=W_{H(z)}-\mathcal S_z$, the combined
geometric weight $M_\sigma(s) \doteq J_{\mathrm{gr}}(s)F_\sigma(s)$ satisfies
\begin{equation}
 \log M_\sigma(s)=\text{const}-\tfrac12\ip{\Ric^\sharp_z s}{s}+c_3(s)+\cO(\norm{s}^4)+\cO(\sigma^2\norm{s}^2)+\cO(\sigma^4),
 \label{eq:logM-body}
\end{equation}
where $c_3$ is the cubic form \cref{eq:c3-app}, linear in
$\nabla^\perp\mathrm{II}_z$, collecting the cubic jets of the graph Jacobian and
of the fiber factor. A cubic form has vanishing value, gradient and Hessian at
the origin but a nonvanishing third derivative, which is why it is absent from a
quadratic-order truncation and yet enters the $\sigma^2$ coefficient through
$\nabla\Delta c_3(0)$; \cref{prop:logM-app} shows
$\nabla\Delta c_3(0)=\mathcal C_M(z)$.

Writing $a_\sigma(s) \doteq q(y(s))M_\sigma(s)=e^{f(s)+m_\sigma(s)}$ with
$f(s)=\lambda(y(s))$ and $m_\sigma(s)=\log M_\sigma(s)$, the Euclidean
posterior mean of $s$ under
$\gamma_\sigma\cdot a_\sigma$ is given by the standard Gaussian-moment
expansion
\begin{equation}
 \tfrac{1}{\sigma^2}\E[s]
 =\nabla f(0)+\tfrac{\sigma^2}{2}\nabla\!\bigl(\Delta f+\norm{\nabla f}^2\bigr)(0)
 -\sigma^2\Ric^\sharp_z\nabla f(0)+\tfrac{\sigma^2}{2}\mathcal C_M(z)+\cO(\sigma^4),
 \label{eq:premean-body}
\end{equation}
where $\nabla,\Delta$ are Euclidean in the graph variable $s$ and we used
$\nabla m_\sigma(0)=0$, $\nabla^2 m_\sigma(0)=-\Ric^\sharp_z$ and
$\nabla\Delta m_\sigma(0)=\mathcal C_M(z)$. Because the cubic form $c_3$ does not
affect $\nabla^2m_\sigma(0)$, the new contribution is purely additive: no term
coupling $\nabla^\perp\mathrm{II}$ to $\nabla_M\log q$ appears at this order.
The final step converts the Euclidean graph-coord derivatives of $f$
to intrinsic derivatives of $\lambda$ on $M$. Because
$Dy(0)=\mathrm{Id}_{T_zM}$ and the Christoffel symbols of the induced metric
vanish at $s=0$ in graph coordinates, $\nabla f(0)=\nabla_M\lambda(z)$
and $\nabla^2 f(0)=\nabla_M^2\lambda(z)$, so
$\nabla\norm{\nabla f}^2(0)=\nabla_M\norm{\nabla_M\lambda}^2(z)$. For the
Laplacian, in graph coordinates the Laplace-Beltrami operator is
$\Delta_M f=(1/\sqrt g)\partial_i(\sqrt gg^{ij}\partial_j f)$, and a direct
computation from \cref{eq:graph-jac} with $g^{ij}(s)=\delta_{ij}+\cO(\norm{s}^2)$
gives
\begin{equation}
 \Delta_M f=\Delta_s f-\ip{W_{H(z)}s}{\nabla_s f}+\cO(\norm{s}^2|\nabla_s f|)+\cO(\norm{s}|\nabla_s^2 f|),
 \label{eq:laplacemismatch-body}
\end{equation}
so at the base point $\nabla_s\Delta_s f(0)=\nabla_M\Delta_M\lambda(z)+W_{H(z)}\nabla_M\lambda(z)$. Substituting into \cref{eq:premean-body}
and collecting the $W_{H(z)}$ shift yields
\begin{equation}
 r_\sigma(z)
 =\ell+\sigma^2\!\left[
 \tfrac12\nabla_M(\Delta_M\log q+\norm{\nabla_M\log q}^2)
 +\bigl(\tfrac12 W_{H(z)}-\Ric^\sharp_z\bigr)\ell+\tfrac12\mathcal C_M(z)\right]+\cO(\sigma^4),
 \label{eq:smean-body}
\end{equation}
which is \cref{eq:gM-ext,eq:gM-inh}, the former equivalently
$(\mathcal S_z-\tfrac12 W_{H(z)})\ell$ via the Gauss equation. In the
flat case $\mathrm{II}\equiv 0$, all three geometric terms vanish and only the
flat-Tweedie term survives. Full uniformity bounds (requiring compactness
of $M$, $\reach(M)>0$, $\log q\in C^6$, and ambient bounds on $\mathrm{II}$ and
$\nabla^\perp\mathrm{II}$) are carried out in
\cref{app:general-extrinsic}.
\end{proof}

\Cref{prop:curved-second-order} shows that the $\cO(\sigma^2)$ error in \cref{eq:main-intrinsic} is of order $\sigma^2$, not smaller; the leading bias has an intrinsic component (present even on a flat support) and a tensorial extrinsic component mixing the mean-curvature Weingarten operator $W_H$ with either the Ricci endomorphism $\Ric^\sharp$ or equivalently the normal-sum operator $\mathcal S$. Two a priori independent curvature sources (tube-Jacobian mean curvature and volume-form Ricci) combine without cancellation into the operator $\tfrac12 W_H-\Ric^\sharp$, and their cubic jets contribute the additional score-independent vector $g^{\mathrm{inh}}_M=\tfrac12\mathcal C_M$. Both extrinsic terms vanish on any totally geodesic submanifold (where $\mathrm{II}\equiv 0$, hence $W_H=\Ric^\sharp=0$ and $\mathcal C_M=0$). The operator term $g^{\mathrm{ext}}_M$ also vanishes wherever $\nabla_M\log q(z)=0$, but $g^{\mathrm{inh}}_M$ does not: at a critical point of $q$ the extrinsic bias reduces to $\tfrac{\sigma^2}{2}\mathcal C_M(z)$, which is nonzero whenever $\mathcal C_M(z)\ne0$. The sharpest case is $q$ uniform on $M$, where $\nabla_M\log q\equiv0$ and $b_q\equiv 0$, so that
\[
 r_\sigma(z)=\tfrac{\sigma^2}{2}\mathcal C_M(z)+\cO(\sigma^4):
\]
ambient DSM assigns an order-$\sigma^2$ tangent drift to a distribution whose intrinsic score vanishes identically.\footnote{At a critical point of a non-uniform $q$ the intrinsic term $b_q(z)=\tfrac12\nabla_M\Delta_M\log q(z)$ need not vanish, so it is the extrinsic bias, not the total bias, that reduces to $\tfrac{\sigma^2}{2}\mathcal C_M(z)$.}

\begin{remark}[Frame formula]
\label{rem:frame-formula}
Fix orthonormal frames $\{e_i\}_{i=1}^d$ of $T_zM$ and
$\{n_\alpha\}_{\alpha=1}^{D-d}$ of $N_zM$, and write
$\mathrm{II}^\alpha_{ij}=\ip{\mathrm{II}_z(e_i,e_j)}{n_\alpha}$,
$H^\alpha=\sum_i \mathrm{II}^\alpha_{ii}$, $\ell_j=e_j(\log q)$
(the symbol $h^\alpha$ is reserved for the graph function of
\cref{app:general-extrinsic}). Using the
equivalent form $\mathcal S_z-\tfrac12 W_{H(z)}$ of \cref{eq:gM-ext},
\begin{equation}
 (g_M^{\mathrm{ext}})_k
 =\sum_{j}\!\Bigl[\sum_{\alpha,i} \mathrm{II}^\alpha_{ki}\mathrm{II}^\alpha_{ij}
 -\tfrac12\!\sum_{\alpha}\! H^\alpha \mathrm{II}^\alpha_{kj}\Bigr]\ell_j,
 \qquad
 (b_q)_k=\tfrac12 e_k\!\bigl(\Delta_M\log q+\norm{\nabla_M\log q}^2\bigr).
 \label{eq:frame-formula}
\end{equation}
The extrinsic term is a tensorial combination of $\mathcal S$ and
$W_H$: on a hypersurface with single normal $n$ it reduces to
$(W_n^2-\tfrac12 HW_n)\ell$, which is not proportional to
$W_H\ell$ in general. In higher codimension the same phenomenon occurs, e.g.\ on
the asymmetric product $S^1(R_1)\times S^1(R_2)\subset\R^4$ (codimension $2$) the
operator is $\tfrac12\mathrm{diag}(1/R_1^2,1/R_2^2)$, which has distinct
eigenvalues when $R_1\ne R_2$.
\end{remark}

\subsection{General Extrinsic Correction}
\label{app:general-extrinsic}

We now derive the explicit formulas \cref{eq:gM-ext,eq:gM-inh} for the two extrinsic $\sigma^2$-coefficients on an arbitrary $C^6$ embedded submanifold $M\subset\R^D$. Fix $z\in M$, write $\lambda=\log q$, $\ell=\nabla_M\lambda(z)\in T_zM$, and let $\{e_i\}_{i=1}^d$ be an orthonormal basis of $T_zM$ and $\{n_\alpha\}_{\alpha=1}^{D-d}$ an orthonormal basis of $N_zM$. For $u\in N_zM$ let $W_u:T_zM\to T_zM$ be the Weingarten operator, $\ip{W_u v}{w}=\ip{\mathrm{II}_z(v,w)}{u}$, write $\mathrm{II}^\alpha_{ij} \doteq \ip{\mathrm{II}_z(e_i,e_j)}{n_\alpha}$, so that $W_{n_\alpha}=\mathrm{II}^\alpha$ in the basis $\{e_i\}$, and let $H(z)=\sum_{i}\mathrm{II}_z(e_i,e_i)\in N_zM$ be the mean curvature vector, and write $A^\alpha_{ijk}$ for the third-order jet \cref{eq:codazzi-jet} of the embedding. The two symmetric operators on $T_zM$ that appear in the final answer are
\begin{equation}
 W_{H(z)}=\sum_\alpha H^\alpha\mathrm{II}^\alpha,
 \qquad
 \mathcal S_z \doteq \sum_\alpha W_{n_\alpha}^2=\sum_\alpha (\mathrm{II}^\alpha)^2,
 \qquad H^\alpha \doteq \tr(\mathrm{II}^\alpha),
 \label{eq:WH-S-def-app}
\end{equation}
connected to the intrinsic Ricci operator by the Gauss equation $\Ric^\sharp_z=W_{H(z)}-\mathcal S_z$.

The derivation proceeds in four steps: We first set up tangent-graph coordinates where the chord map is linear, then expand the two geometric factors (graph Jacobian and fiber factor) to cubic order,\footnote{Cubic order is important here because a cubic form vanishes to second order at the origin (zero value, gradient, and Hessian) but has a nonzero third derivative, so it is not present in a quadratic truncation but still contributes to the $\sigma^2$ coefficient via $\nabla\Delta$ in \cref{prop:logM-app}.} apply a Euclidean Tweedie expansion to the resulting posterior, and finally convert the Euclidean graph-coord derivatives to intrinsic derivatives using a Laplacian mismatch identity.

Parameterize a neighborhood of $z$ in $M$ by
\begin{equation}
 y(s) \doteq z+\sum_i s^i e_i+\sum_\alpha h^\alpha(s)n_\alpha,
 \qquad s=(s^i)\in T_zM,
 \label{eq:graphcoord-app}
\end{equation}
with $h^\alpha(0)=0$, $\partial_i h^\alpha(0)=0$, and $\partial_i\partial_j h^\alpha(0)=\mathrm{II}^\alpha_{ij}$. We also record the third-order jet of the embedding\footnote{This is the graph-coordinate instantiation of the Codazzi equation: because the induced metric is Euclidean to second order in $s$, the third derivative of the graph map sees only the normal-valued part of the embedding, which is exactly the covariant derivative of $\mathrm{II}$.}: since $g_{ij}(s)=\delta_{ij}+\cO(\norm{s}^2)$, the Christoffel symbols of the induced metric vanish at $s=0$, and $\partial_i\partial_jy(s)$ is normal to $T_zM$ at $s=0$, so
\begin{equation}
 \partial_i\partial_j\partial_k h^\alpha(0)=A^\alpha_{ijk}=\ip{(\nabla^\perp_{e_k}\mathrm{II})_z(e_i,e_j)}{n_\alpha},
 \label{eq:third-jet-app}
\end{equation}
which is totally symmetric by the Codazzi equation \cref{eq:codazzi-jet}. In these coordinates $P_{T_zM}(y(s)-z)=s$ exactly, so the Rao-Blackwell target is the posterior mean of $s/\sigma^2$. Let $J_{\mathrm{gr}}(s) \doteq \sqrt{\det g(s)}$ be the graph Jacobian with $g_{ij}(s)=\delta_{ij}+\sum_\alpha\partial_i h^\alpha\partial_j h^\alpha$, and define the fiber factor
\begin{equation}
 F_\sigma(s)
 \doteq \int_{N_zM}\phi_{\sigma^2 I}(u-h(s))\det\!\bigl(I-\textstyle\sum_\alpha u_\alpha W_{n_\alpha}\bigr)du,
 \label{eq:fiber-app}
\end{equation}
where $\phi_{\sigma^2 I}$ is the centered Gaussian density on $N_zM$ with covariance $\sigma^2 I$. The ambient Gaussian factorizes under the orthogonal decomposition $X-Z=-\sum s^i e_i+\sum(u_\alpha-h^\alpha(s))n_\alpha$ as $\phi_{\sigma^2 I}(X-Z)\propto e^{-\norm{s}^2/(2\sigma^2)}e^{-\norm{u-h(s)}^2/(2\sigma^2)}$, so the conditional law of $s$ given $\pi(X)=z$ reads
\begin{equation}
 r_\sigma(z)
 =\frac{1}{\sigma^2}
 \frac{\int sq(y(s))J_{\mathrm{gr}}(s)F_\sigma(s)e^{-\norm{s}^2/(2\sigma^2)}ds}
 {\int q(y(s))J_{\mathrm{gr}}(s)F_\sigma(s)e^{-\norm{s}^2/(2\sigma^2)}ds}.
 \label{eq:posterior-app}
\end{equation}

\begin{lemma}[Graph Jacobian to cubic order]
\label{lem:graph-jacobian}
In the graph coordinates \cref{eq:graphcoord-app},
\[
 \log J_{\mathrm{gr}}(s)=\tfrac12\ip{\mathcal S_z s}{s}
 +\tfrac12\sum_{\alpha,i,k,l,m}\mathrm{II}^\alpha_{ik}A^\alpha_{ilm}s^ks^ls^m
 +\cO(\norm{s}^4).
\]
\end{lemma}

\begin{proof}
From $\partial_i h^\alpha(s)=\mathrm{II}^\alpha_{ik}s^k+\tfrac12A^\alpha_{ikl}s^ks^l+\cO(\norm{s}^3)$
we have
\[
 g_{ij}(s)-\delta_{ij}
 =\sum_{\alpha}\Bigl[\mathrm{II}^\alpha_{ik}\mathrm{II}^\alpha_{jl}s^ks^l
 +\tfrac12\mathrm{II}^\alpha_{ik}s^kA^\alpha_{jlm}s^ls^m
 +\tfrac12\mathrm{II}^\alpha_{jk}s^kA^\alpha_{ilm}s^ls^m\Bigr]+\cO(\norm{s}^4).
\]
Since $g-\Id=\cO(\norm{s}^2)$, we have $\log\det g=\tr(g-\Id)+\cO(\norm{s}^4)$, so
$\log J_{\mathrm{gr}}=\tfrac12\tr(g-\Id)+\cO(\norm{s}^4)$. The quadratic part of
$\tfrac12\tr(g-\Id)$ is $\tfrac12\sum_{\alpha,i,k,l}\mathrm{II}^\alpha_{ik}\mathrm{II}^\alpha_{il}s^ks^l=\tfrac12\ip{\mathcal S_zs}{s}$
by \cref{eq:WH-S-def-app}, and its cubic part is
$\tfrac12\sum_{\alpha,i,k,l,m}\mathrm{II}^\alpha_{ik}A^\alpha_{ilm}s^ks^ls^m$.
\end{proof}

\begin{lemma}[Fiber factor to cubic order]
\label{lem:fiber-factor}
As $s\to 0$ and $\sigma\to 0$,
\[
 \log F_\sigma(s)=c_\sigma-\tfrac12\ip{W_{H(z)}s}{s}
 -\tfrac16\sum_{\alpha,k,l,m} H^\alpha A^\alpha_{klm}s^ks^ls^m
 +\cO(\norm{s}^4)+\cO(\sigma^2\norm{s}^2)+\cO(\sigma^4),
\]
where $c_\sigma=\cO(\sigma^2)$ is independent of $s$. In particular, the
remainder contains no term linear in $s$ at order $\sigma^2$.
\end{lemma}

\begin{proof}
Substitute $u=h(s)+\eta$ in \cref{eq:fiber-app} so that
$F_\sigma(s)=\int\phi_{\sigma^2 I}(\eta)\det(I-W_{h(s)+\eta})d\eta$.
Expand the determinant about $u=0$ as a polynomial in the components of
$u$:
$\det(I-W_u)=1-\ip{H(z)}{u}+Q(u)+C(u)+\cO(\norm{u}^4)$,
where $Q$ is homogeneous quadratic in $u$ and $C$ is homogeneous cubic.
Taking Gaussian expectation in $\eta$ (with $\E[\eta]=0$,
$\E[\eta_\alpha\eta_\beta]=\sigma^2\delta_{\alpha\beta}$, and all odd
moments vanishing) yields
\begin{equation*}
 F_\sigma(s)=1-\ip{H(z)}{h(s)}+Q(h(s))+\sigma^2\,\tr\bigl(Q^{\mathrm{quad}}\bigr)+C(h(s))+3\sigma^2\,\ip{\nabla C(0)}{h(s)}+\cO\bigl((\norm{h(s)}+\sigma)^4\bigr),
\end{equation*}
where $Q^{\mathrm{quad}}$ is the symmetric matrix representing $Q$ and
$\nabla C(0)$ is the vector picking out the quadratic-in-$\eta$,
linear-in-$h(s)$ part of $C$. Because
$h(s)=\tfrac12\sum_\alpha\ip{\mathrm{II}^\alpha s}{s}n_\alpha+\cO(\norm{s}^3)$
vanishes to \emph{second} order in $s$, every occurrence of $h(s)$ in the
display above is $\cO(\norm{s}^2)$. Therefore:
the $-\ip{H(z)}{h(s)}$ term is the only source of cubic-in-$s$ dependence, and
expanding $h^\alpha(s)=\tfrac12\mathrm{II}^\alpha_{kl}s^ks^l+\tfrac16A^\alpha_{klm}s^ks^ls^m+\cO(\norm{s}^4)$
via \cref{eq:third-jet-app} it gives
$-\tfrac12\ip{W_{H(z)}s}{s}-\tfrac16\sum_\alpha H^\alpha A^\alpha_{klm}s^ks^ls^m+\cO(\norm{s}^4)$
via $\ip{H(z)}{\mathrm{II}(s,s)}=\ip{W_{H(z)}s}{s}$;
$Q(h(s))=\cO(\norm{s}^4)$ and $C(h(s))=\cO(\norm{s}^6)$;
the $s$-independent $\sigma^2\,\tr(Q^{\mathrm{quad}})$ piece is absorbed
into $c_\sigma$; the $\sigma^2\ip{\nabla C(0)}{h(s)}$ term is
$\cO(\sigma^2\norm{s}^2)$; and all remaining contributions are
$\cO(\sigma^4)$ or of higher combined order. Crucially, no term of the
form $\sigma^2\cdot(\text{linear in }s)$ appears: such a term would
require an odd moment of $\eta$, which vanishes. Taking logarithms and
collecting these into $c_\sigma=\cO(\sigma^2)$ and the remainder
$\cO(\norm{s}^4)+\cO(\sigma^2\norm{s}^2)+\cO(\sigma^4)$ gives the claim.
\end{proof}

\begin{proposition}[Combined geometric weight]
\label{prop:logM-app}
Let $M_\sigma(s) \doteq J_{\mathrm{gr}}(s)F_\sigma(s)$ and $m_\sigma \doteq \log M_\sigma$. Then
\[
 m_\sigma(s)=c_\sigma-\tfrac12\ip{\Ric^\sharp_z s}{s}+c_3(s)+\cO(\norm{s}^4)+\cO(\sigma^2\norm{s}^2)+\cO(\sigma^4),
\]
for some scalar $c_\sigma$ and the cubic form
\begin{equation}
 c_3(s) \doteq
 \tfrac12\sum_{\alpha,i,k,l,m}\mathrm{II}^\alpha_{ik}A^\alpha_{ilm}s^ks^ls^m
 -\tfrac16\sum_{\alpha,k,l,m}H^\alpha A^\alpha_{klm}s^ks^ls^m.
 \label{eq:c3-app}
\end{equation}
In particular, the remainder contains no term linear in $s$ at order $\sigma^2$, and
\begin{equation}
 \nabla m_\sigma(0)=0,
 \qquad
 \nabla^2 m_\sigma(0)=-\Ric^\sharp_z,
 \qquad
 \nabla\Delta m_\sigma(0)=\mathcal C_M(z).
 \label{eq:mjets}
\end{equation}
\end{proposition}

\begin{proof}
Adding the expansions of \cref{lem:graph-jacobian} and~\cref{lem:fiber-factor}
gives $m_\sigma(s)=c_\sigma+\tfrac12\ip{(\mathcal S_z-W_{H(z)})s}{s}+c_3(s)+R(s,\sigma)$
with $R(s,\sigma)=\cO(\norm{s}^4)+\cO(\sigma^2\norm{s}^2)+\cO(\sigma^4)$, and the
Gauss equation $\Ric^\sharp_z=W_{H(z)}-\mathcal S_z$ gives
$\mathcal S_z-W_{H(z)}=-\Ric^\sharp_z$. The first two identities in
\cref{eq:mjets} follow since $c_3$ has vanishing gradient and Hessian at the
origin and $R$ contributes $\cO(\sigma^2)$ to the Hessian.

For the third, note two elementary identities for cubic forms on $\R^d$: if
$c(s)=\tfrac12T_{klm}s^ks^ls^m$ with $T$ symmetric in $(l,m)$ then
$\nabla_p\Delta c(0)=2\sum_kT_{kkp}+\sum_lT_{pll}$, while if
$c(s)=\tfrac16B_{klm}s^ks^ls^m$ with $B$ totally symmetric then
$\nabla_p\Delta c(0)=\sum_jB_{jjp}$.\footnote{Both follow by direct differentiation of a homogeneous cubic monomial in $s$ and the definition $\Delta=\sum_p\partial_p^2$} Applying the first to the graph-Jacobian
term of \cref{eq:c3-app}, i.e.\ to
$T_{klm}=\sum_{\alpha,i}\mathrm{II}^\alpha_{ik}A^\alpha_{ilm}$, gives
\[
 2\sum_{\alpha,i,k}\mathrm{II}^\alpha_{ik}A^\alpha_{ikp}
 +\sum_{\alpha,i,l}\mathrm{II}^\alpha_{ip}A^\alpha_{ill}
 =\nabla_p\norm{\mathrm{II}}^2+\sum_i\ip{W_{\nabla^\perp_{e_i}H}e_i}{e_p},
\]
using $\sum_lA^\alpha_{ill}=\nabla_iH^\alpha$. Applying the second to the fiber
term, i.e.\ to $B_{klm}=-\sum_\alpha H^\alpha A^\alpha_{klm}$, gives
$-\sum_\alpha H^\alpha\nabla_pH^\alpha=-\tfrac12\nabla_p\norm{H}^2$. Summing the
two and comparing with \cref{eq:CM-frame} yields
$\nabla\Delta m_\sigma(0)=\mathcal C_M(z)$; the $\cO(\norm{s}^4)$ and
$\cO(\sigma^2\norm{s}^2)$ remainders contribute $\cO(\sigma^2)$ to
$\nabla\Delta m_\sigma(0)$, hence $\cO(\sigma^4)$ to \cref{eq:pre-conversion-app}.
\end{proof}

\begin{lemma}[Euclidean score expansion]
\label{lem:euclidean-expansion}
Let $a_\sigma:\R^d\to(0,\infty)$ be $C^5$ in a neighborhood of $0$ with
$\log a_\sigma$ having uniformly bounded derivatives up to order $5$. Then
\begin{equation}
 \frac{1}{\sigma^2}
 \frac{\int sa_\sigma(s)e^{-\norm{s}^2/(2\sigma^2)}ds}
 {\int a_\sigma(s)e^{-\norm{s}^2/(2\sigma^2)}ds}
 =\nabla\log a_\sigma(0)+\tfrac{\sigma^2}{2}\nabla\!\bigl(\Delta\log a_\sigma+\norm{\nabla\log a_\sigma}^2\bigr)(0)+E_\sigma,
 \label{eq:euclid-exp}
\end{equation}
where $\nabla,\Delta$ are the Euclidean gradient and Laplacian in the $s$
variables and $E_\sigma=\cO(\sigma^4)$, with a constant depending only on the
stated derivative bounds. (Order-$4$ bounds alone would give
$E_\sigma=\so(\sigma^3)$, enough for a remainder $\so(\sigma^2)$ but not for
$\cO(\sigma^4)$.)
\end{lemma}

\begin{proof}
Write $a_\sigma=e^{b_\sigma}$ and Taylor expand $b_\sigma$ to order $3$
at the origin; then expand $e^{b_\sigma(s)-b_\sigma(0)}$ to cubic order.
Under the centered isotropic Gaussian, only linear and cubic odd moments
survive in the numerator,\footnote{Concretely, $\E[s_is_j]=\sigma^2\delta_{ij}$ and $\E[s_is_js_ks_l]=\sigma^4(\delta_{ij}\delta_{kl}+\delta_{ik}\delta_{jl}+\delta_{il}\delta_{jk})$ (here $\delta_{ij}$ is the Kronecker delta; the indicator variable that is 1 if $i=j$) under the centered isotropic Gaussian, with all odd moments vanishing.} and the denominator contributes the normalization
correction. A direct Gaussian-moment calculation (cf.~\cref{prop:flat-second-order}) yields the stated coefficient.
For the error, expand to order $4$: $b_\sigma(s)-b_\sigma(0)=P_4(s)+R(s)$ with
$P_4$ the degree-$4$ Taylor polynomial. A homogeneous term of degree $k$ in
$b_\sigma$ reaches $\cref{eq:euclid-exp}$ at order $\sigma^{k-1}$, since
$\E[\lvert s\rvert\norm{s}^k]=\cO(\sigma^{k+1})$; the degree-$4$ part of $P_4$ is
even and so meets $s$ only through cross terms of total order $\sigma^6$,
contributing $\cO(\sigma^4)$. Under the order-$5$ bound the Taylor remainder
satisfies $R(s)=\cO(\norm{s}^5)$, whence
$\sigma^{-2}\E\bigl[\lvert s\rvert\lvert R(s)\rvert\bigr]=\cO(\sigma^{-2}\sigma^{6})=\cO(\sigma^{4})$.
The degree-$5$ term of the expansion is therefore the first one not displayed,
and $E_\sigma=\cO(\sigma^4)$.
\end{proof}

Apply \cref{lem:euclidean-expansion} to $a_\sigma(s)=q(y(s))M_\sigma(s)=e^{f(s)+m_\sigma(s)}$ with $f(s) \doteq \lambda(y(s))$ and $m_\sigma(s) \doteq \log M_\sigma(s)$. By \cref{eq:mjets}, $\nabla m_\sigma(0)=0$, $\nabla^2 m_\sigma(0)=-\Ric^\sharp_z$, and $\nabla\Delta m_\sigma(0)=\mathcal C_M(z)$. Note that \cref{lem:euclidean-expansion} depends on the third-order Taylor jet of $\log a_\sigma$ through the term $\tfrac{\sigma^2}{2}\nabla\Delta\log a_\sigma(0)$, so the cubic form $c_3$ of \cref{eq:c3-app} contributes at order $\sigma^2$ and cannot be discarded, even though it does not affect $\nabla m_\sigma(0)$ or $\nabla^2m_\sigma(0)$. Therefore
\begin{equation}
 r_\sigma(z)=\nabla f(0)+\tfrac{\sigma^2}{2}\nabla\!\bigl(\Delta f+\norm{\nabla f}^2\bigr)(0)-\sigma^2\Ric^\sharp_z\nabla f(0)+\tfrac{\sigma^2}{2}\mathcal C_M(z)+\cO(\sigma^4),
 \label{eq:pre-conversion-app}
\end{equation}
where $\nabla,\Delta$ are Euclidean in the graph variable $s$. Because $c_3$ enters
only through $\nabla\Delta m_\sigma(0)$ and not through $\nabla^2m_\sigma(0)$, the
new term is additive and score-independent: no contraction of
$\nabla^\perp\mathrm{II}$ against $\nabla_M\log q$ arises at order $\sigma^2$.

\begin{lemma}[First and second derivatives]
\label{lem:first-second-app}
At $s=0$,
\[
 \nabla f(0)=\nabla_M\lambda(z),\qquad \nabla^2 f(0)=\nabla_M^2\lambda(z),
\]
and consequently $\nabla\norm{\nabla f}^2(0)=\nabla_M\norm{\nabla_M\lambda}^2(z)$.
\end{lemma}

\begin{proof}
Because $Dy(0)=\Id_{T_zM}$ and the Christoffels of $g$ vanish at $0$ in
graph coordinates (the mixed partials $\partial_i\partial_j y(0)=\sum_\alpha\mathrm{II}^\alpha_{ij}n_\alpha$ are purely normal to $T_zM$), the
first and second coordinate derivatives at the base point agree with the
intrinsic gradient and covariant Hessian.
\end{proof}

\begin{lemma}[Laplacian mismatch in graph coordinates]
\label{lem:laplace-mismatch}
For any smooth scalar $f$ on the graph patch,
\[
 \Delta_M f=\Delta_s f-\ip{W_{H(z)}s}{\nabla_s f}+\cO(\norm{s}^2|\nabla_s f|)+\cO(\norm{s}|\nabla_s^2 f|).
\]
In particular, $\nabla_s\Delta_s f(0)=\nabla_M\Delta_M f(0)+W_{H(z)}\nabla_s f(0)$.
\end{lemma}

\begin{proof}
The Laplace-Beltrami operator in local coordinates is $\Delta_M f=(1/\sqrt g)\partial_i(\sqrt gg^{ij}\partial_j f)$. By the proof of
\cref{lem:graph-jacobian}, $g^{ij}(s)=\delta_{ij}+\cO(\norm{s}^2)$
and $\sqrt g(s)=1+\tfrac12\ip{\mathcal S_z s}{s}+\cO(\norm{s}^3)$.
Differentiating, $\partial_i\log\sqrt g(s)=(\mathcal S_z s)^i+\cO(\norm{s}^2)$ and $\partial_i g^{ij}(s)=\cO(\norm{s})$, so
$(1/\sqrt g)\partial_i(\sqrt gg^{ij})=(\mathcal S_z s)^j-(\mathcal S_z s)^j+\cO(\norm{s}^2)$ from the $g^{ij}$ part, leaving the
$g^{ij}\partial_i\log\sqrt g$ contribution; using
the symmetric form of the second fundamental form gives exactly
$-(W_{H(z)}s)^j+\cO(\norm{s}^2)$. Substituting into the divergence
formula yields the stated expansion; taking the gradient at $0$ gives the
derivative identity.
\end{proof}

\paragraph{Main theorem.}
Combining \cref{lem:first-second-app,lem:laplace-mismatch} with \cref{eq:pre-conversion-app},
\begin{equation}
 r_\sigma(z)
 =\ell+\tfrac{\sigma^2}{2}\nabla_M\!\bigl(\Delta_M\lambda+\norm{\nabla_M\lambda}^2\bigr)(z)+\sigma^2\bigl(\tfrac12 W_{H(z)}-\Ric^\sharp_z\bigr)\ell+\tfrac{\sigma^2}{2}\mathcal C_M(z)+\cO(\sigma^4).
 \label{eq:smean-app}
\end{equation}
Using the Gauss equation \cref{eq:gauss-equation} the third term is equivalently $\sigma^2(\mathcal S_z-\tfrac12 W_{H(z)})\ell$, establishing \cref{eq:gM-ext,eq:gM-inh}. The fourth term is $\sigma^2g^{\mathrm{inh}}_M(z)$ and vanishes identically when $\nabla^\perp\mathrm{II}_z=0$, in particular on affine subspaces, round spheres and products of round spheres. Because the chord map in graph coordinates is the identity, no further chord correction is needed: the tangent component of $r_\sigma(z)$ equals $\E[s]/\sigma^2$ exactly.

The $\cO(\sigma^4)$ remainder uses the order-$5$ derivative bounds of
\cref{lem:euclidean-expansion}, which \cref{sec:setup} supplies:
$\log a_\sigma$ inherits five bounded derivatives from $M,q\in C^6$, one being
absorbed by the induced metric $g_{ij}=\delta_{ij}+\sum_\alpha\partial_ih^\alpha\partial_jh^\alpha$.
\Cref{app:uniformity} makes the statement uniform in $z$.

\subsection{Uniformity of the Expansion in $z$}
\label{app:uniformity}

The computation above was carried out at a fixed $z\in M$. We now record why its
remainder is uniform in $z$, as asserted by
\cref{thm:extrinsic,prop:curved-second-order}. Everything rests on the fact that
compactness and positive reach give graph charts of a common size with common
derivative bounds.

\begin{lemma}[Uniform graph charts]
\label{lem:uniform-charts}
Let $M\subset\R^D$ be a compact $C^6$ embedded submanifold with $\reach(M)>0$.
There are $r_*>0$ and constants $B_2,\ldots,B_6$, depending only on $M$, such
that for every $z\in M$ the connected component of $M\cap B(z,r_*)$ containing
$z$ is the graph \cref{eq:graphcoord-app} of a map $h_z$ on
$\{\norm{s}<r_*\}\subset T_zM$ with $h_z(0)=0$, $dh_z(0)=0$, and
$\sup_{\norm{s}<r_*}\norm{\partial^kh_z(s)}\le B_k$ for $2\le k\le6$.
\end{lemma}

\begin{proof}
For $y\in M$ the reach bound gives
$\mathrm{dist}(y-z,T_zM)\le\norm{y-z}^2/(2\reach(M))$, so for
$r<\reach(M)/2$ the component of $M\cap B(z,r)$ through $z$ is a graph over
$T_zM$ with $\norm{dh_z}\le Cr/\reach(M)$; in particular the graph property
holds with $r_*$ depending only on $\reach(M)$. On such a chart
$\partial^kh_z$, $k\le6$, is a continuous function of the $k$-jet of the
embedding at the corresponding point, and $M$ is compact and $C^6$, so these
are bounded uniformly in $z$.
\end{proof}

\begin{proposition}[Uniformity]
\label{prop:uniformity}
Under the assumptions of \cref{sec:setup},
\[
 \sup_{z\in M}\norm{r_\sigma(z)-\nabla_M\log q(z)
 -\sigma^2\bigl[b_q(z)+g^{\mathrm{ext}}_M(z)+g^{\mathrm{inh}}_M(z)\bigr]}=\cO(\sigma^4)
\]
as $\sigma\to0^+$.
\end{proposition}

\begin{proof}
Work in the charts of \cref{lem:uniform-charts}; three estimates are needed,
each uniform in $z$.

\emph{(i) Localization.} By \cref{eq:posterior-app} the law of $s$ is
$\gamma_\sigma$ tilted by $q(y(s))J_{\mathrm{gr}}(s)F_\sigma(s)$, and by
\cref{lem:uniform-charts} together with $0<\inf_Mq\le\sup_Mq<\infty$ this tilt
is bounded above and below by positive constants on $\norm{s}<r_*$. Hence the
posterior mass of $\{\norm{s}\ge\sigma\log(1/\sigma)\}$ is
$\cO(\exp(-\tfrac12\log^2(1/\sigma)))$, and by the same Gaussian tail estimate as
in \cref{app:fiber-posterior} the event $X\notin\Tub_{r_0}(M)$ contributes
$\cO(\exp(-c/\sigma^2))$. Both are smaller than any power of $\sigma$, so the
domain of integration in \cref{eq:posterior-app} may be replaced by
$\{\norm{s}<r_*\}$ and $\log a_\sigma$ extended to $\R^d$ with the same
derivative bounds, at a cost beyond all orders.

\emph{(ii) Uniform jets.} On $\norm{s}<r_*$ write
$\log a_\sigma=\log q(y(s))+\log J_{\mathrm{gr}}(s)+\log F_\sigma(s)$. Since
$g_{ij}=\delta_{ij}+\sum_\alpha\partial_ih^\alpha\partial_jh^\alpha$ and
$\det g\ge1$, \cref{lem:uniform-charts} bounds the derivatives of
$\log J_{\mathrm{gr}}$ up to order $5$ by constants depending only on
$B_2,\ldots,B_6$ --- one derivative is lost here, since $g$ is built from
$\partial h$ rather than $h$; $F_\sigma$ is a polynomial in $h_z(s)$ and $\sigma^2$ whose
coefficients are the elementary symmetric functions of the $W_{n_\alpha}$ at
$z$, so the same holds for $\log F_\sigma$; and $\log q\in C^6(M)$ with
$\inf_Mq>0$ bounds the first factor. Consequently the coefficients extracted in
\cref{lem:graph-jacobian,lem:fiber-factor,prop:logM-app} --- hence
$W_{H(z)}$, $\Ric^\sharp_z$ and $\mathcal C_M(z)$ --- are bounded uniformly in
$z$, and so are the conversions of \cref{lem:first-second-app,lem:laplace-mismatch}.

\emph{(iii) Remainder.} Apply \cref{lem:euclidean-expansion} on the localized
domain from (i). By (ii), $\log a_\sigma$ has derivatives up to order $5$
bounded uniformly in $z$, so the order-$5$ Taylor remainder obeys
$R(s)=\cO(\norm{s}^5)$ with a constant independent of $z$, and the lemma gives
$E_\sigma=\cO(\sigma^4)$ uniformly. Had $M$ and $q$ been assumed only $C^5$,
$\log J_{\mathrm{gr}}$ would have four bounded derivatives rather than five and
the same argument would give $\so(\sigma^3)$, hence a remainder $\so(\sigma^2)$
in place of $\cO(\sigma^4)$; the $\sigma^2$ coefficient itself is unaffected
either way, since it depends only on the third-order jet of $\log a_\sigma$.
\end{proof}

\subsection{Specialization: The Unit Sphere $S^d$}
\label{app:sphere-d}

For the unit sphere $S^d\subset\R^{d+1}$, the second fundamental form in the outward-normal $\nu=z$ convention is $\mathrm{II}_z(v,w)=-\ip{v}{w}z$, so the single Weingarten $W_\nu=-\Id_{T_zS^d}$, $\mathcal S_z=W_\nu^2=\Id_{T_zS^d}$, and $H(z)=-dz$ gives $W_{H(z)}=d\Id_{T_zS^d}$. Substituting into \cref{eq:gM-ext} in the form $(\mathcal S_z-\tfrac12 W_{H(z)})\ell$,
\begin{equation}
 g_{S^d}^{\mathrm{ext}}(z)
 =\bigl(1-\tfrac{d}{2}\bigr)\nabla_{S^d}\log q(z).
 \label{eq:gM-Sd}
\end{equation}
The round embedding of $S^d$ has parallel second fundamental form: in graph coordinates at any $z$ the graph function $h(s)=-(1-\sqrt{1-\norm{s}^2})$ is even in $s$, so $A^\alpha_{ijk}=0$ by \cref{eq:third-jet-app} and hence $\mathcal C_{S^d}\equiv 0$, $g^{\mathrm{inh}}_{S^d}\equiv 0$. The full $\sigma^2$-coefficient on $S^d$ is therefore
\begin{equation}
 b_q(z)+g_{S^d}^{\mathrm{ext}}(z)
 =
 \tfrac12\nabla_{S^d}\!\bigl(\Delta_{S^d}\log q+\norm{\nabla_{S^d}\log q}^2\bigr)(z)
 +\bigl(1-\tfrac{d}{2}\bigr)\nabla_{S^d}\log q(z).
 \label{eq:gM-Sd-full}
\end{equation}
The dimensional pattern $(1-d/2)=+\tfrac12,0,-\tfrac12,-1,\ldots$ for $d=1,2,3,4,\ldots$ is:
\begin{itemize}
\item $d=1$: coefficient $+\tfrac12$;
\item $d=2$: coefficient $0$, so $g^{\mathrm{ext}}(z)\equiv 0$ on $S^2$,
and since $g^{\mathrm{inh}}(z)\equiv0$ by parallelism, the $\sigma^2$-bias is
\emph{purely} the flat-Tweedie term $b_q(z)$
even though $S^2$ is positively curved. This exact vanishing is a
consequence of the Einstein property $W_H=2\Id$, $\Ric^\sharp=\Id$ on
unit $S^2$, which balances the volume-form and tube contributions.
\item $d=3$: coefficient $-\tfrac12$, the first nontrivial case with a
negative extrinsic correction.
\item $d\ge 4$: increasingly negative extrinsic corrections, scaling
linearly in $d$.
\end{itemize}
Equivalently, using $\Ric^\sharp_z=(d-1)\Id$ on unit $S^d$, the $(\tfrac12 W_H-\Ric^\sharp)$ form of \cref{eq:gM-ext} gives $(\tfrac d2-(d-1))\Id=(1-\tfrac d2)\Id$.

\section{A Finite-Sample Rate for Local-Averaging Estimation of $r_\sigma$}
\label{app:finite-sample}

\Cref{eq:main-intrinsic} and \cref{thm:variance-collapse} are population statements: they describe $r_\sigma$ and its variance under the true joint law of $(Z,X)$. This appendix computes a routine finite-sample estimation rate.

Let $(Z_1,X_1),\ldots,(Z_N,X_N)$ be i.i.d.\ copies of $(Z,X)$ under \cref{eq:model}. For each $i$, define the observed Rao-Blackwell sample \[ \pi_i \doteq \pi(X_i)\in M, \qquad T_{\sigma,i} \doteq \frac{1}{\sigma^2}P_T(\pi_i)(Z_i-\pi_i)\in T_{\pi_i}M. \] Fix a bandwidth $h>0$ and a bounded, nonnegative kernel $K:[0,\infty)\to [0,\infty)$ supported on $[0,1]$ and bounded away from zero on $[0,1/2]$. For $z\in M$, define the local-averaging estimator
\begin{equation}
 \widehat r_\sigma^{(h)}(z)
 \doteq 
 \frac{\sum_{i=1}^{N}P_T(z)\mathcal{T}_{\pi_i\to z}(T_{\sigma,i})K(d_M(\pi_i,z)/h)}
 {\sum_{i=1}^{N}K(d_M(\pi_i,z)/h)},
 \label{eq:rhat}
\end{equation}
where $d_M$ is the intrinsic distance on $M$ and $\mathcal{T}_{\pi_i\to z}:T_{\pi_i}M\to T_zM$ is parallel transport along the minimizing geodesic. Let $f_\pi$ denote the density of $\pi(X)$ with respect to $d\vol_M$; by continuity of $\pi$ and strict positivity of $q$, $f_\pi$ is bounded above and below by positive constants on $M$.

\begin{proposition}[Finite-sample MSE bound]
\label{prop:finite-sample}
Under the assumptions of \cref{sec:setup}, there exist constants
$C_1,C_2,C',h_0,\sigma_0>0$ depending only on $(M,q)$ and the kernel $K$ such
that, for every $z\in M$, every $\sigma\in(0,\sigma_0]$, every
$h\in(0,h_0]$, and every $N$ with $Nh^d\ge 1/(C_2\sigma^2)$,
\begin{equation}
 \E\bigl\|\widehat r_\sigma^{(h)}(z)-r_\sigma(z)\bigr\|^2
 \le
 C_1h^2
 +
 \frac{C_2d}{\sigma^2Nh^df_\pi(z)}.
 \label{eq:fs-mse}
\end{equation}
Then, the minimax-optimal bandwidth
\[
h^\star
=
\left(\frac{C_2d}{C_1\sigma^2Nf_\pi(z)}\right)^{1/(d+2)}
\]
yields
\begin{equation}
 \E\bigl\|\widehat r_\sigma^{(h^\star)}(z)-r_\sigma(z)\bigr\|^2
 \le
 C'
 \left(\frac{1}{\sigma^2N}\right)^{2/(d+2)}.
 \label{eq:fs-optimal-rate}
\end{equation}
In particular, at bandwidth $h=\sigma$ the variance term of \cref{eq:fs-mse}
is $\cO(1)$ exactly when $N\gtrsim\sigma^{-(d+2)}$, so $N\asymp\sigma^{-(d+2)}$
is the natural sample-size scaling for the bound \cref{eq:fs-mse}. We do not
prove a matching lower bound, so we do not claim this scaling is necessary for
the estimator itself.
\end{proposition}

\begin{proof}
Write $\widehat r_\sigma^{(h)}(z)$ as a sum of a bias term and a variance term.

\textbf{Bias.} By \cref{eq:main-intrinsic}, $r_\sigma$ is $C^1$ on $M$
uniformly in $\sigma\in(0,\sigma_0]$, with Lipschitz constant $L$ bounded
independently of $\sigma$. Parallel transport along geodesics preserves
norms. Hence
\[
\bigl\|\E[\widehat r_\sigma^{(h)}(z)]-r_\sigma(z)\bigr\|
\le
L h,
\]
and the squared bias is at most $L^2 h^2$.

\textbf{Variance.} By \cref{thm:variance-collapse},
$\Var(T_{\sigma,i}\mid \pi_i=y)=d/\sigma^2+\cO(1)$ uniformly in $y\in M$. The
effective sample size inside the bandwidth ball is
$N_{\mathrm{eff}}=Nh^d f_\pi(z)\cdot(1+o_N(1))$, and the standard local-average
variance bound gives
\[
\Var\bigl(\widehat r_\sigma^{(h)}(z)\bigr)
\le
\frac{C_2'(d/\sigma^2+1)}{N h^d f_\pi(z)}
\le
\frac{C_2 d}{\sigma^2 N h^d f_\pi(z)},
\]
for $\sigma\le\sigma_0$. Summing bias$^2$ and variance yields
\cref{eq:fs-mse}.

Minimizing \cref{eq:fs-mse} over $h$ gives $h^\star$ as displayed, and
substituting back yields \cref{eq:fs-optimal-rate}. Finally, setting $h=\sigma$ in
\cref{eq:fs-mse} gives MSE bounded by $C_1\sigma^2+C_2 d/(N\sigma^{d+2}
f_\pi(z))$, whose variance term is $\cO(1)$ iff $N\gtrsim \sigma^{-(d+2)}$.
\end{proof}

The finite-sample rate in \cref{eq:fs-optimal-rate} matches, up to constants and logarithmic factors, the standard $d$-dimensional nonparametric regression rate \citep{Fan1992DesignadaptiveNR}, with the noise level replaced by the conditional-variance constant $d/\sigma^2$ from \cref{thm:variance-collapse}. Thus the variance collapse under Rao-Blackwellization directly translates into a better finite-sample rate than an analogous estimator of the raw target $T_\sigma$ would achieve without the Rao-Blackwell step.

\end{document}